\documentclass{article} 
\usepackage[left=1.5in, right=1.5in, bottom=1.3in, headsep=.2in]{geometry}
\usepackage{times}
\usepackage{hyperref}
\usepackage{url}
\usepackage{bm}
\usepackage{amsmath,amsthm,amssymb}
\usepackage{graphicx}
\usepackage{enumerate}
\usepackage{caption}
\usepackage{subcaption}

\newtheorem{theorem}{Theorem}
\newtheorem{lemma}[theorem]{Lemma}

\newcommand{\hk}{\hat{\kappa}}

\newcommand{\hpp}{P_{\hat{V}_k}^{\perp}}

\newcommand{\pp}{P_{V_k}^{\perp}}
\newcommand{\bx}{\bm{x}}
\newcommand{\by}{\bm{y}}
\newcommand{\yqk}{\hat{\bm{y}}_{q,k}}
\newcommand{\oyqk}{\bm{y}_{q,k}^*}
\newcommand{\oyq}{\bm{y}_{q}^*}
\newcommand{\bq}{\bm{q}}

\title{Revisiting Kernelized Locality-Sensitive Hashing \\for Improved Large-Scale Image Retrieval}

\author{
Ke Jiang, Qichao Que, Brian Kulis \\
Department of Computer Science and Engineering\\
The Ohio State University\\
\texttt{\{jiangk,que,kulis\}@cse.ohio-state.edu} \\
}

\begin{document}

\maketitle


\begin{abstract}
We present a simple but powerful reinterpretation of kernelized locality-sensitive hashing (KLSH), a general and popular method developed in the vision community for performing approximate nearest-neighbor searches in an arbitrary reproducing kernel Hilbert space (RKHS).  Our new perspective is based on viewing the steps of the KLSH algorithm in an appropriately projected space, and has several key theoretical and practical benefits.  First, it eliminates the problematic conceptual difficulties that are present in the existing motivation of KLSH.  Second, it yields the first formal retrieval performance bounds for KLSH.  Third, our analysis reveals two techniques for boosting the empirical performance of KLSH.  We evaluate these extensions on several large-scale benchmark image retrieval data sets, and show that our analysis leads to improved recall performance of at least 12\%, and sometimes much higher, over the standard KLSH method.
\end{abstract}

\section{Introduction}

Similarity search (or nearest neighbor (NN) search) for large databases plays a critical role in a number of important vision applications including content-based image and video retrieval. Usually, the data are represented in a high-dimensional feature space, and the number of objects in the database can scale to the billions in modern applications. As such, fast indexing and search is a vital component to many large-scale retrieval systems. 

A key theoretical and practical breakthrough for the similarity search problem was the development of locality-sensitive hashing (LSH) \cite{lsh,plsh,simhash}, which relies on Gaussian random projections for Euclidean distance and can provably retrieve \textit{approximate} nearest neighbors in time that grows sublinearly in the number of database items.  In the vision community, LSH has long been employed as one of the core methods for large-scale retrieval~\cite{fast,shift_lsh,klsh,weak-supervised,scalable-kernel,multi_klsh,supervised,kpca_lsh}.  Unfortunately, in some cases, image comparison criteria are based on functions other than the simple Euclidean distance between corresponding image feature vectors, which makes LSH inapplicable in several settings. Some foundational work has been done to extend LSH to kernels satisfying particular conditions, such as hashing for shift-invariant kernels in~\cite{shift_lsh} based on random Fourier features~\cite{random_sink}. More generally, Kulis and Grauman~\cite{klsh} proposed a technique called kernelized LSH (KLSH) for approximate nearest neighbor searches with arbitrary kernels, thus extending LSH to situations where only kernel function evaluations are possible. The main idea behind KLSH is to approximate the necessary Gaussian random projections in the kernel space using an appropriate random combination of items from the database, based on an application of the central limit theorem.

Since its publication, KLSH has been used extensively in the computer vision community, and related hashing methods have been built from the KLSH foundations \cite{weak-supervised, scalable-kernel, multi_klsh, supervised}; however, KLSH still suffers from some important drawbacks. First, while Kulis and Grauman show that the central limit theorem ensures that the approximate random projections constructed become true Gaussian random projections as the number of sampled database items gets larger, no bounds are explicitly given to clarify the tradeoff between accuracy and runtime. Even worse, the approach that KLSH uses--that of applying a random projection with a $\mathcal{N}(0,I)$ vector--is conceptually inappropriate in an infinite-dimensional kernel space since, as we will discuss later, no such canonical Gaussian distribution even exists.

In this paper, we present a simple yet powerful reinterpretation of KLSH, which we describe in Section \ref{sec:interpretation}. This new perspective gracefully resolves the ``infinite Gaussian" issue and provides us with the first explicit performance bounds to clearly demonstrate tradeoffs between runtime and retrieval accuracy.  Crucially, this tradeoff also reveals two potential techniques which boost the empirical performance of vanilla KLSH. In particular, we show how to modify KLSH to obtain improvements in recall performance of at least 12\%, and sometimes much higher, on all the benchmarks examined in Section~\ref{sec:experiment}.

\subsection{Related work and our contributions}
There is limited existing theoretical analysis \cite{multi_klsh} of KLSH based on Nystr\"om approximation bounds \cite{generalization_nystrom}. However, this analysis only examines the average error between the original kernel function values and the approximations made by KLSH, and does not provide any bounds on retrieval performance. Moreover, as we will discuss in Section~\ref{sec:interpretation}, there is a subtle difference between KLSH and the Nystr\"om method, rendering the aforementioned analysis problematic.  Further, we will demonstrate in Section~\ref{sec:experiment} that KLSH bears advantages over the Nystr\"om method when the number of database items selected to approximate kernel functions is relatively small.

There have been conflicting views about the comparison of KLSH and LSH after applying kernel PCA~\cite{kpca} to the data. For example, some work \cite{klsh} has concluded that KLSH has a clear performance edge over KPCA+LSH, while these results are contradicted by the empirical analysis in \cite{kpca_lsh,aska} which demonstrated that LSH after a KPCA projection step shows a significant improvement over KLSH. We will see in Section~\ref{sec:interpretation} that these two seemingly disparate methods are \textit{equivalent} (up to how the random vectors are drawn in the two approaches), and the performance gap observed in practice is largely \textit{only} due to the choice of parameters. Although \cite{kpca_lsh} gives some error analysis for the LSH after a PCA projection step using the Cauchy-Schwarz inequality, no explicit performance bounds are proved. Thus, it fails to show the interesting tradeoffs and retrieval bounds that we derive in Section~\ref{sec:bounds}.

Recently, there has been work on kernel approximation-based visual search methods. Asymmetric sparse kernel approximations \cite{aska} aim to approximate the nearest neighbor search with an asymmetric similarity score computed from $m$ randomly selected landmarks. It has shown excellent empirical performance with $m=8192$. Kernelized random subspace hashing (KRSH) \cite{krsh} attempts to randomly generate the orthogonal bases for an $m$-dimensional subspace in kernel space. Then classical hashing schemes are employed on the projection to this subspace. These approaches may be viewed as variants of the Nystr\"om method; we note that the authors of~\cite{krsh} were able to provide a preservation bound on inter-vector angles and showed better angle preservation than KLSH.

Our main contribution can be summarized as threefold. First, we provide a new interpretation of KLSH, which not only provides a firmer theoretical footing but also resolves issues revolving around comparisons between KLSH and LSH after projection via kernel PCA. Second, we are able to derive the first formal retrieval bounds for KLSH, demonstrating a tradeoff similar to the classic bias-variance tradeoff in statistics. Lastly and most importantly, our analysis reveals two potential techniques for boosting the performance of standard KLSH.  We successfully validate these techniques on large-scale benchmark image retrieval datasets, showing at least a 12\% improvement in recall performance across all benchmarks.

\section{Background: LSH for Similarities and KLSH}
\label{sec:klsh}
Assume that the database is a set of $n$ samples $\{\bm{x}_1,\ldots,\bm{x}_n\}\in\mathbb{R}^d$. Given a query $\bm{q}\in\mathbb{R}^d$ and a user-defined kernel function $\kappa(\cdot,\cdot) = \langle \Phi(\cdot),\Phi(\cdot)\rangle$ with the feature map $\Phi: \mathbb{R}^d\to\mathcal{H}$, where $\mathcal{H}$ is the implicit reproducing kernel Hilbert space (RKHS), we are interested in finding the most similar item in the database to the query  $\bm{q}$ with respect to $\kappa(\cdot, \cdot)$, i.e., $\mbox{argmax}_i \kappa(\bm{q},\bm{x}_i)$. 

LSH is a general technique for constructing and applying hash functions to the data such that two similar objects are more likely to be hashed together~\cite{lsh,plsh}.  When the hash functions are binary, and $b$ hash functions are employed, this results in a projection of the data into a $b$-dimensional binary (Hamming) space.  Note that there are several possible LSH schemes, including non-binary hashes, but we will focus mainly on binary hashing in this paper.  One advantage to binary hashing is that nearest neighbor queries in the Hamming space can be implemented very quickly; tree-based data structures can be used to find approximate nearest neighbors in the Hamming space in time sub-linear in the number of data points~\cite{simhash}, and even an exact nearest neighbor computation can be performed extremely quickly in the Hamming space.  

In order to meet the locality-sensitive requirement for similarity functions, each hash function $h$ should satisfy \cite{simhash}:
\[
	\text{Pr}[h(\bm{x}_i)=h(\bm{x}_j)] = \kappa(\bm{x}_i,\bm{x}_j).
\]
Here, we only consider normalized kernel functions $\kappa(\cdot,\cdot)\in[0,1]$; for un-normalized kernels, our results can be applied after normalization via $\kappa(\bm{x},\bm{y})/\sqrt{\kappa(\bm{x},\bm{x})\cdot \kappa(\bm{y},\bm{y})}$. Given valid hash families, the query time for retrieving ($1+\epsilon$)-nearest neighbors is bounded by $O(n^{1/(1+\epsilon)})$ for the Hamming distance \cite{lsh, simhash}. For the linear kernel $\kappa(\bm{x}_i,\bm{x}_j)=\bm{x}_i^T\bm{x}_j$ (i.e. $\Phi(\bm{x}) = \bm{x}$) for normalized histograms, Charikar~\cite{simhash} showed that a hash family can be constructed by rounding the output of the product with a random hyperplane:
\begin{equation}
	h_{\bm{r}}(x) = \left\{ \begin{array}{ll}
	1, & \text{if }\bm{r}^T\bm{x}\ge 0 \\
	0, & \text{otherwise}
	\end{array} \right. , 
	\label{eqn:lsh}
\end{equation} 
where $\bm{r}\in\mathbb{R}^d$ is a random vector sampled from the standard multivariate Gaussian distribution (i.e., from ${\mathcal N}(0,I))$. This can be directly extended to kernels having known explicit representations with dimension $d_\Phi<\infty$. However, this is not the case for many commonly-used kernels in vision applications.

In order to deal with arbitrary kernels, KLSH attempts to mimic this technique by drawing approximate Gaussian random vectors in the RKHS via the central-limit theorem (CLT). The key advantage to this approach is that the resulting hash function computation can be accomplished solely using kernel function evaluations.


Considering $\{\Phi(\bm{x}_1),\ldots,\Phi(\bm{x}_t)\}$ as $t$ realizations of random variable $\Phi(X)$ with known mean $\mu$ and covariance operator $C$, the classical CLT~\cite{Feller} ensures that the random vector $C^{-1/2}\tilde{z}_t = C^{-1/2}[\sqrt{t}(\frac{1}{t}\sum_{i=1}^{t}\Phi(\bm{x}_i)-\mu)]$ converges to a standard Gaussian random vector as $t\rightarrow\infty$.
Therefore, the hash family \eqref{eqn:lsh} can be approximated by:
\begin{equation}
	h(\Phi(x)) = \left\{ \begin{array}{ll}
	1, & \text{if }\Phi(x)^T C^{-1/2}\tilde{\bm{z}_t}\ge 0 \\
	0, & \text{otherwise}
	\end{array} \right. . 
	\label{eqn:klsh}
\end{equation}

In practice, the mean $\mu$ and the covariance matrix $C$ of the data are not known and must be estimated through a random set $S=\{\hat{\bm{x}}_{1},\ldots,\hat{\bm{x}}_{m}\}$ from the database. Choosing the $t$ random samples used in the CLT approximation from $S$ ($t<m$), \cite{klsh} showed that \eqref{eqn:klsh} has the convenient form
\begin{equation}
	h(\Phi(\bm{x}))=\text{sign}(\sum_{i=1}^{m}\bm{w}(i)\kappa(\bm{x},\hat{\bm{x}}_{i})),
	\label{eqn:final_klsh}
\end{equation}
where $\bm{w}=\bar{K}^{-1/2}\bm{e}_{S_t}$ with $\bar{K}$ the $m\times m$ centered kernel matrix formed by $\{\hat{\bm{x}}_{1},\ldots,\hat{\bm{x}}_{m}\}$ and $\bm{e}_{S_t}$ an $m\times1$ vector with ones at the entries corresponding to the $t$ samples. Note that some constant scaling terms have been dropped without changing the hash function evaluation.

The validity of KLSH relies heavily on the central limit theorem. One crucial question that is not addressed in~\cite{klsh} is the existence of $\mathcal{N}(0,I_{\infty})$ in the case where the kernel function is based on an infinite-dimensional embedding (such as the Gaussian kernel).  Unfortunately, there is no such canonical Gaussian distribution in an RKHS, as given by the following lemma.
\begin{lemma}
\cite{GPs} A Gaussian distribution with covariance operator $C$ in a Hilbert space exists if and only if, in an appropriate base, $C$ has a diagonal form with non-negative eigenvalues and the sum of these eigenvalues is finite.
\label{lemma:clt}
\end{lemma}
As implied by Lemma \ref{lemma:clt}, the convergence to the standard Gaussian in an infinite-dimensional Hilbert space is not grounded, as the eigenvalues of the covariance operator sum to infinity.\footnote{Note that the central limit theorem does still apply in Hilbert spaces, but the covariance operators must always have finite trace.}  As such, the motivation for KLSH is problematic at best and, at worst, could render KLSH inappropriate for many of the retrieval settings for which it was specifically designed. At the same time, KLSH has shown solid empirical performance on kernels associated with infinite-dimensional $\mathcal{H}$ \cite{klsh}. How can we explain the discrepancy between the empirical performance and the lack of a solid theoretical motivation?  We resolve these issues in the next section.

\section{A New Interpretation of KLSH}
\label{sec:interpretation}
In the following, we will provide a simple but powerful reinterpretation of KLSH, which will allow us to circumvent the aforementioned issues with infinite-dimensional $\mathcal{H}$.  In particular, we will show that KLSH may be viewed precisely as KPCA+LSH, except that the Gaussian vectors drawn for LSH are drawn via the CLT in the KPCA projected space.

\vspace{1mm}
\noindent\textbf{KLSH as Explicit Embedding}.
Let us take a deeper look at the hash function \eqref{eqn:klsh}. Utilizing the eigen-decomposition of the covariance $C$, we can write
\begin{align}
	g(\Phi(\bm{x})) &= \sum_{i=1}^{d_{\Phi}}\frac{1}{\sqrt{\lambda_i}}(\bm{v}_i^T\Phi(\bm{x})) \cdot (\bm{v}_i^T\tilde{\bm{z}}_t) \nonumber \\
	&= \sum_{i=1}^{d_{\Phi}}(\bm{v}_i^T\Phi(\bm{x})) \cdot \bigg (\frac{1}{\sqrt{\lambda_i}}\bm{v}_i^T\tilde{\bm{z}}_t \bigg ),
	\label{eqn:ex_embed}
\end{align}
where $h(\Phi(\bm{x})) = \text{sign}[g(\Phi(\bm{x}))]$, $\lambda_1\ge\cdots\ge\lambda_m\ge\cdots\ge0$ are the eigenvalues of $C$ with $\bm{v}_i$ the corresponding eigenvectors. In many situations, the dimension $d_\Phi$ of $\Phi$ is infinite.

If we perform a truncation at $k$, we obtain a finite-dimensional representation for $\Phi$ with $\sum_{i=k+1}^{d_\Phi}\lambda_i$ as the expected approximation error.  The resulting sum in~\eqref{eqn:ex_embed} can be viewed as an inner product between two $k$-dimensional vectors.
 Specifically, the first vector is $(\bm{v}_1^T\Phi(\bm{x}), \bm{v}_2^T\Phi(\bm{x}),\ldots, \bm{v}_k^T\Phi(\bm{x}))$, which is simply the projection of $\Phi(\bm{x})$ onto the subspace spanned by the top $k$ principal components.  For the other $k$-dimensional vector $(\frac{1}{\sqrt{\lambda_1}}\bm{v}_1^T\tilde{\bm{z}}_t,\ldots,\frac{1}{\sqrt{\lambda_k}}\bm{v}_k^T\tilde{\bm{z}}_t)$, we plug in the definition of $\tilde{\bm{z}_t}$ to obtain for each $i$:
 \begin{align}
 	\frac{1}{\sqrt{\lambda_i}}\bm{v}_i^T\tilde{\bm{z}}_t &= \frac{\sqrt{t}}{\sqrt{\lambda_i}} \bigg (\frac{1}{t}\sum_{j\in S}\bm{v}_i^T\Phi(\bm{x}_j) - \bm{v}_i^T\mu \bigg ) \nonumber \\
	&= \frac{\sqrt{t}}{\sqrt{\lambda_i}}(\bar{y}_{i,t} - \mu_i)   \nonumber \\
	&\overset{appr}\sim \mathcal{N}(0,1), \nonumber
\end{align}
where $\bar{y}_{i,t} = \frac{1}{t}\sum_{j\in S}\bm{v}_i^T\Phi(\bm{x}_j)$ (i.e., the $t$ sample average of the projection of $\Phi$ onto the $\bm{v}_i$-direction), $\mu_i$ is the projection of $\mu$ onto the $\bm{v}_i$-direction, and the last approximation comes from the central limit theorem. Since we do not know $\mu$ and $C$ explicitly, we can use plug-in sample estimates.

Under the above interpretation, we see that~\eqref{eqn:ex_embed} after truncation may be viewed as computing an inner product between two $k$-dimensional vectors: the first vector is the data point $\Phi(\bm{x})$ after projecting via KPCA, and the second vector is a Gaussian random vector.  In other words, this sum may be interpreted as first computing KPCA, and then using the CLT to draw random vectors for LSH in the projected space.  Since KLSH uses a sample of $m$ data points from the data set to estimate the covariance and mean, we automatically obtain truncation of the sum as the estimated covariance is at most ($m$-1)-dimensional (it is $(m$-$1)$-dimensional because of the centering operation).   We can therefore see that KLSH performs LSH after projecting to an ($m$-$1$)-dimensional space via principal component analysis in $\mathcal{H}$.  Thus, KLSH is conceptually equivalent to applying LSH after a PCA projection in $\mathcal{H}$. The only difference is that KLSH uses the central limit theorem to approximately draw Gaussian vectors in the projected space, whereas standard KPCA+LSH draws the random vectors directly.  Note that the central limit theorem is known to converge at a rate of $O(t^{-1/2})$, and in practice obtains very good approximations when $t \geq 30$ \cite{Feller,klsh}; as such, results of the two algorithms are within small random variations of each other.  

In summary, we are now able to explain the empirical performance of~\cite{klsh} and also avoid the technical issues with the non-existence of standard Gaussian distributions in infinite-dimensional Hilbert spaces. As we will see, this perspective will lead to a performance bound for KLSH in Section~\ref{sec:bounds}, which also sheds light on simple techniques which could potentially improve the retrieval performance.

\vspace{1mm}
\noindent\textbf{Comparison with the Nystr\"om method.}
As arguably the most popular kernel matrix approximation method, the Nystr\"om method \cite{generalization_nystrom} uses a low-rank approximation to the original kernel matrix in order to reduce the complexity of computing inverses. Although KLSH bears some similarities with the Nystr\"om method as pointed out in \cite{multi_klsh}, here we briefly clarify the differences between these approaches.

Given $m$ anchor points, $\{\hat{\bm{x}}_1,\ldots,\hat{\bm{x}}_m\}$, the Nystr\"om method constructs a rank-$r$ approximation to the kernel matrix over all the database items as 
\[
	\tilde{K}_r = K_{nm}\hat{K}_r^{\dagger}K_{nm}^T,
\]
where $\tilde{K}_r$ is the rank-$k$ approximation to the original kernel matrix $K$, $K_{nm}$ is the $n\times m$ kernel matrix with the $(i,j)$-th entry as $\kappa(\bm{x}_i,\hat{\bm{x}}_j), i=1,\ldots,n,j=1,\ldots,m$, and $\hat{K}_r^{\dagger}$ is the pseudo-inverse of $\hat{K}_r$.
 which is the best rank-$k$ approximation to the $m\times m$ kernel matrix $\hat{K}$ formed by the selected anchor points.
We can write the best rank-$r$ approximation as $\hat{K}_r = \hat{U}_r\hat{D}_r\hat{U}_r^T$, where $\hat{D}_r$ is the $r\times r$ diagonal matrix whose diagonal entries are the leading $r$ eigenvalues of $\hat{K}$, and $\hat{U}_r$ is the $m\times r$ matrix with each column as the corresponding eigenvector of $\hat{K}$. Now, we can derive a vector representation for each data as
\[
	\psi(\bm{x}) = \hat{D}_r^{-1/2}\hat{U}_r^T(\kappa(\bm{x},\hat{\bm{x}}_1), \ldots,\kappa(\bm{x},\hat{\bm{x}}_m))^T.
\]

This is very similar to the representation used in equation \eqref{eqn:final_klsh}, where we can write the vector representation (considering also a rank-$r$ approximation) as
\[
	\phi(\bm{x}) = \bar{D}_r^{-1/2}\bar{U}_r^T(\kappa(\bm{x},\hat{\bm{x}}_1), \ldots,\kappa(\bm{x},\hat{\bm{x}}_m))^T,
\]
where $\bar{D}_r$ and $\bar{U}_r$ are the diagonal matrix with leading $r$ eigenvalues and the $m\times r$ matrix with each column the corresponding eigenvectors of $\bar{K}$, respectively. $\bar{K}$ is the centered version of $\hat{K}$. 

Although they look similar in format, the two representations turn out to yield very different hashing performance. Even though the Nystr\"om method aims to approximate the whole kernel matrix, we point out that the ``centering" operation used by KLSH is essential to give strong performance, especially when $m$ is relatively small.  We will empirically explore this issue further in Section~\ref{sec:experiment}.

\section{Theoretical Analysis}

In this section, we present our main theoretical results, which consist of a performance bound for the KLSH method analogous to results for standard LSH. Perhaps more importantly, our analysis suggests two simple techniques to improve the empirical performance of the KLSH method.

\subsection{A performance bound}
\label{sec:bounds}
We present a theoretical analysis of KLSH via the ``KPCA+LSH" perspective.  We make the assumption that KLSH draws truly random Gaussian vectors in the projected subspace, for simplicity of presentation.  A more refined analysis would also include the error included by approximating the Gaussian random vectors via the central limit theorem.  Such analysis should be possible to incorporate via the Berry-Esseen theorem \cite{Feller}; however, as discussed earlier, in practice the CLT-based random vectors provide sufficient approximations to Gaussian vectors (see also Section 6.6 of \cite{klsh_journal} for a more in-depth empirical comparison). 

We first formulate the setting and assumptions for our main results. Suppose $S = \{\bm{x}_1,\ldots,\bm{x}_n\}$ are $n$ i.i.d. random samples drawn from a probability measure $p$. Given a query $\bq$ and a similarity function $\kappa$, often referred to as a kernel function, we want to find $\oyq=\text{argmax}_{i\in S}\kappa(\bq, \bx_i)$. As is standard in the machine learning literature, when the kernel $\kappa$ is positive semi-definite, it can be (implicitly) associated with a feature map $\Phi$ that maps original data points $\bx$ into a high dimensional or even infinite dimensional feature space $\mathcal{H}$. The kernel function $\kappa(\bx, \by)$ gives the inner product $\langle \Phi(\bx), \Phi(\by) \rangle$ in the feature space. The feature map and $p$ together induce a distribution on $\mathcal{H}$, which has a covariance operator, denoted by $C$. Finally, let $\lambda_1\geq \lambda_2\geq \cdots$ and $\bm{v}_1,\bm{v}_2,\cdots$ be the corresponding eigenvalues and eigenvectors of $C$, respectively.

As we assume that we do not have the explicit formula for the feature map, we project $\Phi(\bx)$ onto a $k$-dimensional subspace $V_k$, which is spanned by the first $k$ principal components of $C$. Empirically, we use a sample of $m$ points to estimate $C$ and $V_k$. Note that after projection, the new kernel function becomes
\[
	\hk(\bx, \by) = \langle P_{\hat{V}_k}(\Phi(\bx)), P_{\hat{V}_k}(\Phi(\by))\rangle,
\]
which is no longer normalized. Here, $\hat{V}_k$ is the sample estimator of $V_k$. To fit into the standard LSH framework, we normalize this kernel by
\[
	\hk_n(\bx, \by) = \frac{\hk(\bx, \by)}{N(\bx)N(\by)},
\]
where $ \quad N(\bx) = \sqrt{\hk(\bx,\bx)} = \|P_{\hat{V}_k}(\Phi(\bx))\| \leq 1$. As points with small $N(\bx)$ may cause instability issues, we simply remove them and only keep points with $N(\bx) \geq 1-\sqrt{\lambda_k} -\eta$ for proper choice of $\eta$.  Note that this step should not affect the search result very much; in fact, we prove in Lemma~\ref{lemma:nx_lower} below that the optimal points $\oyq$ will not be eliminated with high probability.  See the appendix for a proof.

\begin{lemma}
Let $\delta_k = \frac{\lambda_k-\lambda_{k+1}}{2}$ and $\eta=\frac{2}{\delta_k\sqrt{m}}\left(1+\sqrt{\frac{\xi}{2}}\right)$, where $k$ is the total number of chosen principal components and $m$ is the number of points for estimating the eigen-space. With probability at least $1-e^{-\xi}$, for any point $\bx \in \mathcal{X}$, we have
\[
	N(\bx) > 1-\sqrt{\lambda_k}-\eta.
\]
\label{lemma:nx_lower}
\end{lemma}

We can see that the probability of eliminating $\oyq$ is small given the choice of $\eta$. Now we can give our main result in the following theorem.
\begin{theorem}
\label{thm:klsh}
Consider an $n$-sample database $S = \{\bm{x}_1,\ldots,\bm{x}_n\}$ and a query point $\bq$. For any $\epsilon,\xi>0$, with success probability at least $(1-e^{-\xi})/2$ and query time dominated by $O(n^{\frac{1}{1+\epsilon}})$ kernel evaluations, the KLSH algorithm retrieves a nearest neighbor $\hat{\bm{y}}_{\bm{q},k}$ with corresponding bound
\begin{equation}
	\kappa(\bq, \yqk) \ge (1+\epsilon)(1-\sqrt{\lambda_k}-\eta)\kappa(\bq, \oyq) - \epsilon -(2+\epsilon) \left( \sqrt{\lambda_k}+\eta \right)^2,
	\label{eqn:klsh_bound} 
\end{equation}
if we only keep those points with $N(\bx)>1-\sqrt{\lambda_k}-\eta$ for consideration, where $\eta=\frac{2}{\delta_k\sqrt{m}}\left(1+\sqrt{\frac{\xi}{2}}\right)$ and $0<\eta<1-\sqrt{\lambda_k}$.

\end{theorem}
Our proof is given in the appendix, which utilizes existing bounds for kernel PCA~\cite{pca_convergence} and the standard LSH performance bound~\cite{lsh, plsh, simhash}.

\subsection{Discussion and Extensions}
\label{sec:tricks}
\noindent \textbf{Understanding the Bound.} The key ingredient of the bound~\eqref{eqn:klsh_bound} is the error term $\sqrt{\lambda_k}+\frac{2}{\delta_k\sqrt{m}}\left(1+\sqrt{\xi/2}\right)$.  Observe that, as $k$ and $m$ become larger at appropriate rates, both $\sqrt{\lambda_k}$ and $\eta$ will go to zero.  Therefore, as the number of chosen data points and KPCA projections gets large, the bound approaches
\[
\kappa(\bq, \yqk) \ge (1+\epsilon)\kappa(\bq, \oyq) - \epsilon.
\]
Further, as the parameter $\epsilon$ from LSH decreases to zero, this bound guarantees us that the true nearest neighbors will be retrieved.
Also observe that, with a fixed $k$, increasing $m$ will always improve the bound, but the $\sqrt{\lambda_k}$ term will be non-zero and will likely yield retrieval errors. This has been empirically shown in~\cite{kpca_lsh}, namely that the performance of KPCA+LSH saturates when $m$ is large enough, usually in the thousands.

\vspace{1mm}
\noindent \textbf{Low-Rank Extension.} On the other hand, with a fixed $m$, there is a trade-off between decreasing $\lambda_k$ and increasing $\frac{1}{\delta_k}$, similar to the classic bias-variance trade-off in statistics; we expect $\frac{1}{\delta_k}$ to increase as $k$ increases, whereas $\lambda_k$ will decrease for larger $k$.  As a result, for a fixed $m$, the best choice of $k$ may actually not be $k=m-1$, as in the original KLSH algorithm, but could be smaller.   In light of this, we introduce a \textit{low-rank extension of the standard KLSH}:
instead of performing LSH in the  ($m$-$1$)-dimensional subspace of $\mathcal{H}$ as in KLSH, we may actually achieve better results by only projecting into a smaller $r$-dimensional subspace obtained by the top $r$ principal components. Specifically, we replace $\bm{w}$ in equation \eqref{eqn:final_klsh} with
\[
	\bm{w}_r = \bar{K}_{r}^{-1/2}\bm{e}_{S_t},
\]
where $\bar{K}_r$ is the best rank-$r$ approximation to $\bar{K}$. In \cite{kpca_lsh}, the authors recommend to use the same number of principal components as the number of hash bits when applying KPCA+LSH, at least for the $\chi^2$ kernel. We will show in Section~\ref{sec:experiment} that in fact the optimal choice of rank is dependent on the dataset and the kernel, not the number of hash bits. Moreover, the performance of KLSH can be quite sensitive on the choice of $r$.

\vspace{1mm}
\noindent \textbf{Extension via Monotone Transformation.} Another relevant factor is the decaying property of the eigenvalues of the covariance operator $C$ in the induced RKHS $\mathcal{H}$, which not only affects the $\lambda_k$ and $\delta_k$, but also the constant which corresponds to the estimation error. Unlike kernel methods for classification and detection, there is a unique property regarding the use of kernel functions for retrieval: applying a monotone increasing transformation to the original kernel function does not change the ranking provided by the original kernel. 

Thus, we can explore popular transformations that can reward us better retrieval performance. For example, given a kernel $\kappa(\bm{x},\bm{y})$, consider an exponential transformation,
\begin{equation}
	\kappa_s(\bm{x},\bm{y}) = \exp\left( s*(\kappa(\bm{x},\bm{y})-1) \right),
	\label{eqn:transformation}
\end{equation}
where $s>0$ is the scale parameter and $\kappa_s(\bm{x},\bm{y})\in(0,1]$. We can see that the ranking of the nearest neighbors stays the same no matter what value we choose for $s$ as long as $s>0$. However, changing the scaling impacts the eigenvalues of the covariance operator.  In particular, it often slows down the decay with $s>1$ and will eliminate the decay entirely when $s\to\infty$.  
In the context of our bound, that means the $1/\delta_k$ term will increase more slowly.  Moreover, this also reduces the estimation error of the eigen-space. However, with a value of $s$ too large, we also need a very large $k$ in order to keep the $\sqrt{\lambda_k}$ term small.  Thus the scaling must be carefully tuned to balance the trade-off.

\section{Empirical Analysis}
\label{sec:experiment}
We now empirically validate the techniques proposed in Section \ref{sec:tricks}. A comparison with the Nystr\"om method is also reported. 

\subsection{Datasets and evaluation protocols}
We perform our experiments on three benchmark datasets commonly used in large-scale image search comparisons: MIRFLICKR \cite{flickr}, which consists of 1 million randomly sampled Flickr images represented using 150-dimensional edge histogram descriptors; SIFT1M \cite{data}, which has 1 million 128-dimensional local SIFT \cite{sift} descriptors; and GIST1M \cite{data}, which is comprised of 1 million 960-dimensional GIST \cite{gist} descriptors. The query size for all three datasets is 10,000. For MIRFLICKR and GIST1M, all descriptors whose values are all zero were removed from the database.

Throughout the comparisons, we set the parameters of KLSH as follows: we generate a 256-bit hash code, and set $m=1000, t=50$ to form the ``random" projection matrix, which is equivalent to performing kernel PCA with a sample set of size $m=1000$. The choice of the number of bits is to largely suppress the performance variation due to randomness, but the conclusions made here are consistent among different choices of number of bits. 

We consider two popular kernels\footnote{Here, we do not consider fixed-dimensional kernels such as the Hellinger kernel, which are uninteresting for our setting.} from the vision community, namely the $\chi^2$ and intersection kernels for histograms:
\begin{align}
	(\chi^2): & \kappa(\bm{x},\bm{y}) = \sum_{i=1}^{d}\frac{2\bm{x}[i]\bm{y}[i]}{\bm{x}[i]+\bm{y}[i]} \nonumber \\
	(\text{Intersection}): &\kappa(\bm{x},\bm{y}) = \sum_{i=1}^{d} \min(\bm{x}[i], \bm{y}[i]), \nonumber
\end{align}
where $\bm{x}[i]$ is the $i$-th entry of $\bm{x}$. We perform exhaustive nearest neighbors search and evaluate the quality of retrieval using the Recall@$R$ measure, which is the proportion of query vectors for which the nearest neighbor is ranked in the first $R$ positions. This measure indicates the fraction of queries for which the nearest neighbor is retrieved correctly if a short list of $R$ is verified in the original space. Note here, $R=100$ only represents $0.01\%$ of the database items. We focus on this measure since it gives a direct measure of the nearest neighbor search performance.

Note that we are deliberately not comparing to other hashing schemes such as semi-supervised hashing methods or optimization-based methods; our goal is to demonstrate how our analysis can be used to improve results of KLSH.  Existing work on KLSH and variants has considered comparisons between KLSH and other techniques~\cite{klsh} and, given space restrictions, we do not focus on such comparisons here.

\begin{figure*}[t]
\centering
\def\arraystretch{0.4}
\setlength{\tabcolsep}{3pt}
\begin{tabular}{ccc}
	\includegraphics[width=0.32\textwidth]{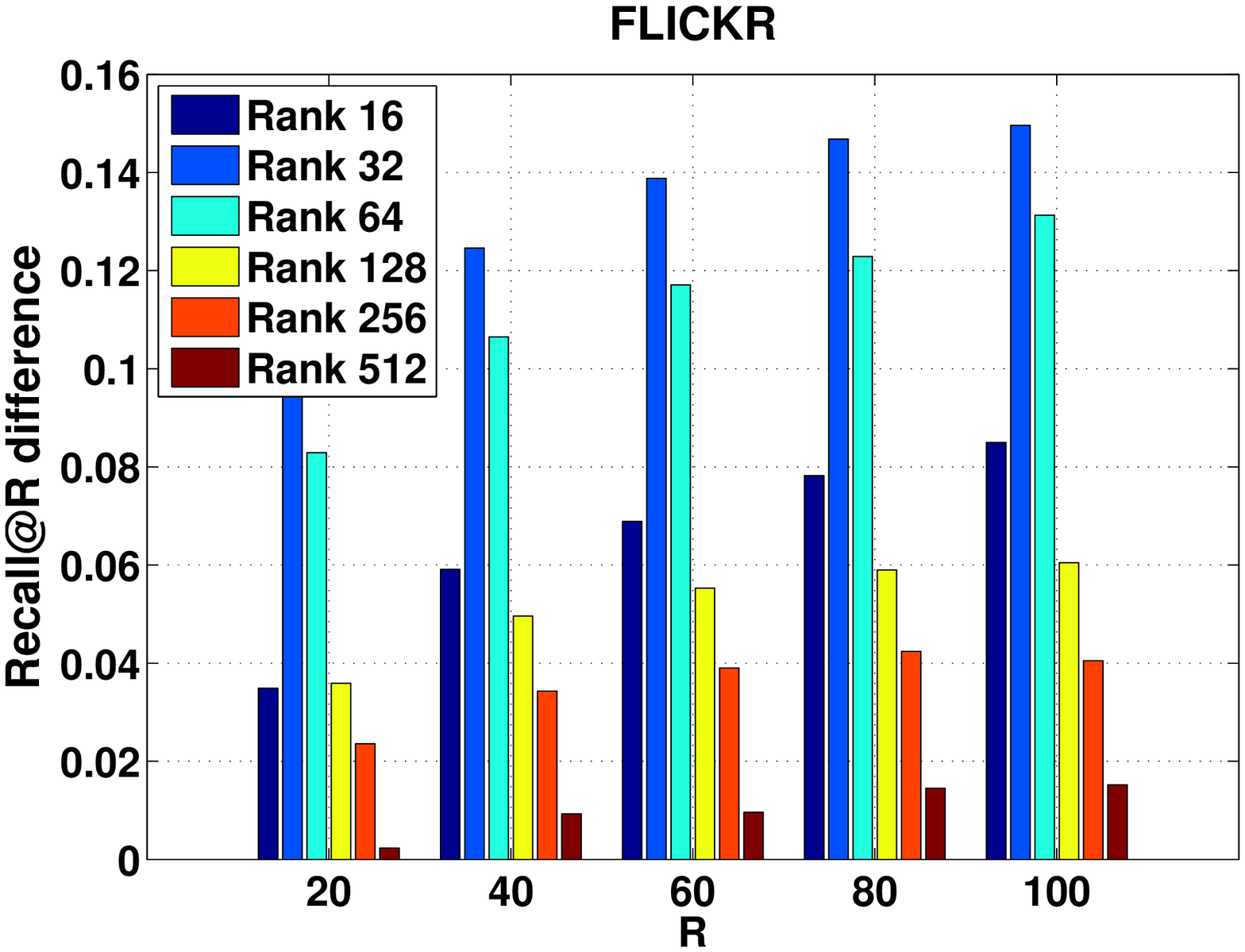} &
	\includegraphics[width=0.32\textwidth]{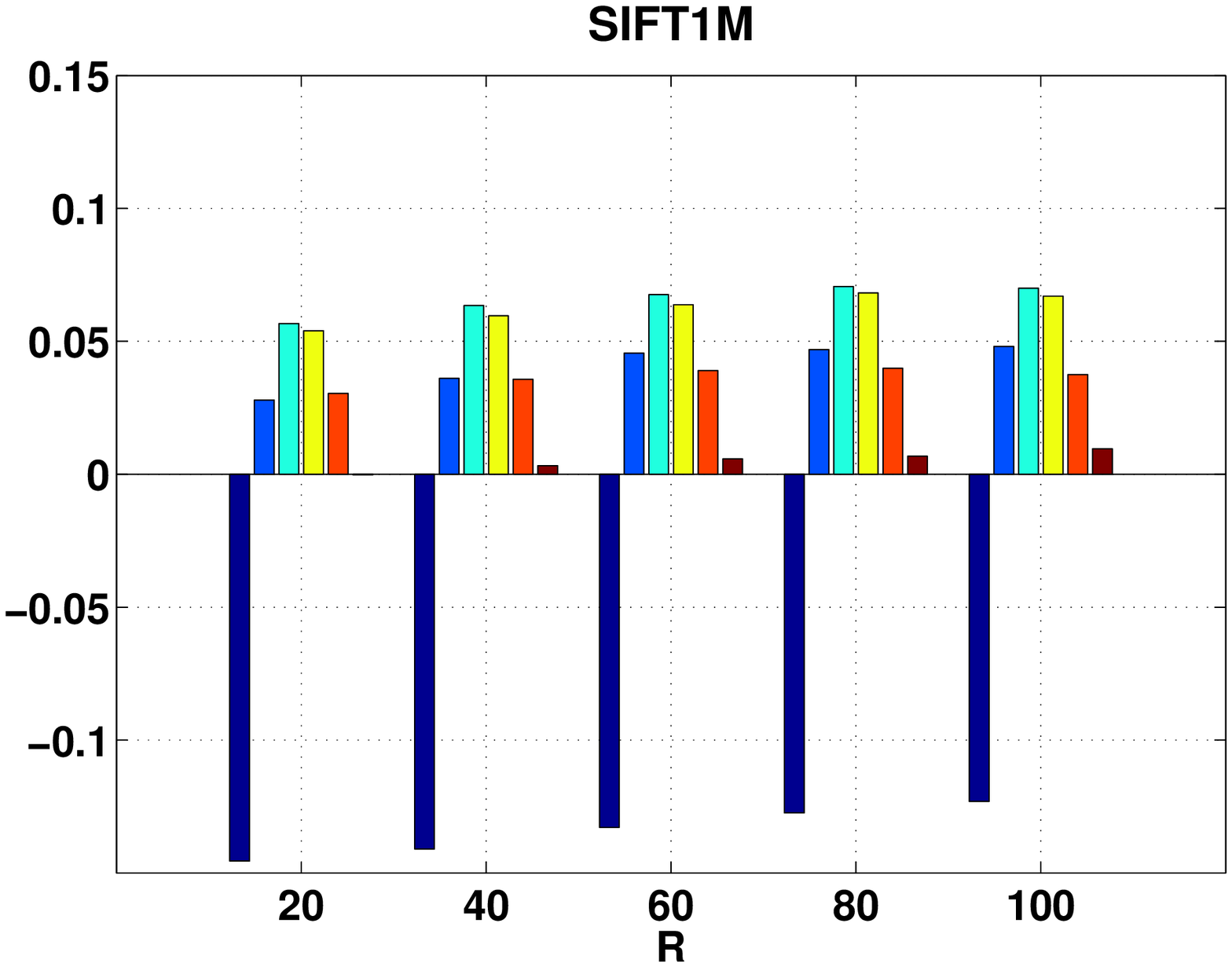} &
	\includegraphics[width=0.32\textwidth]{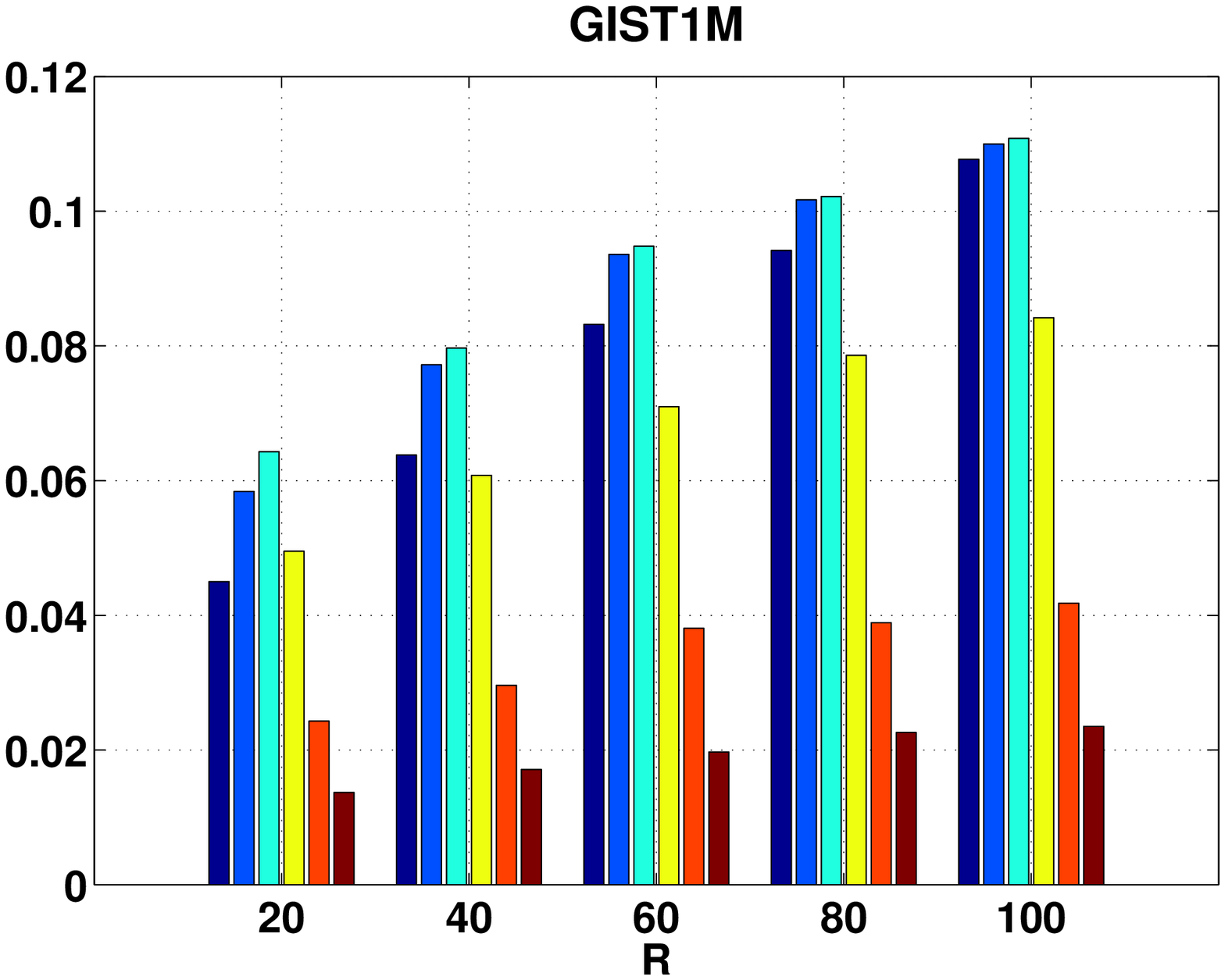} \\
	 & \text{(a) $\chi^2$ kernel} & \\
	\includegraphics[width=0.32\textwidth]{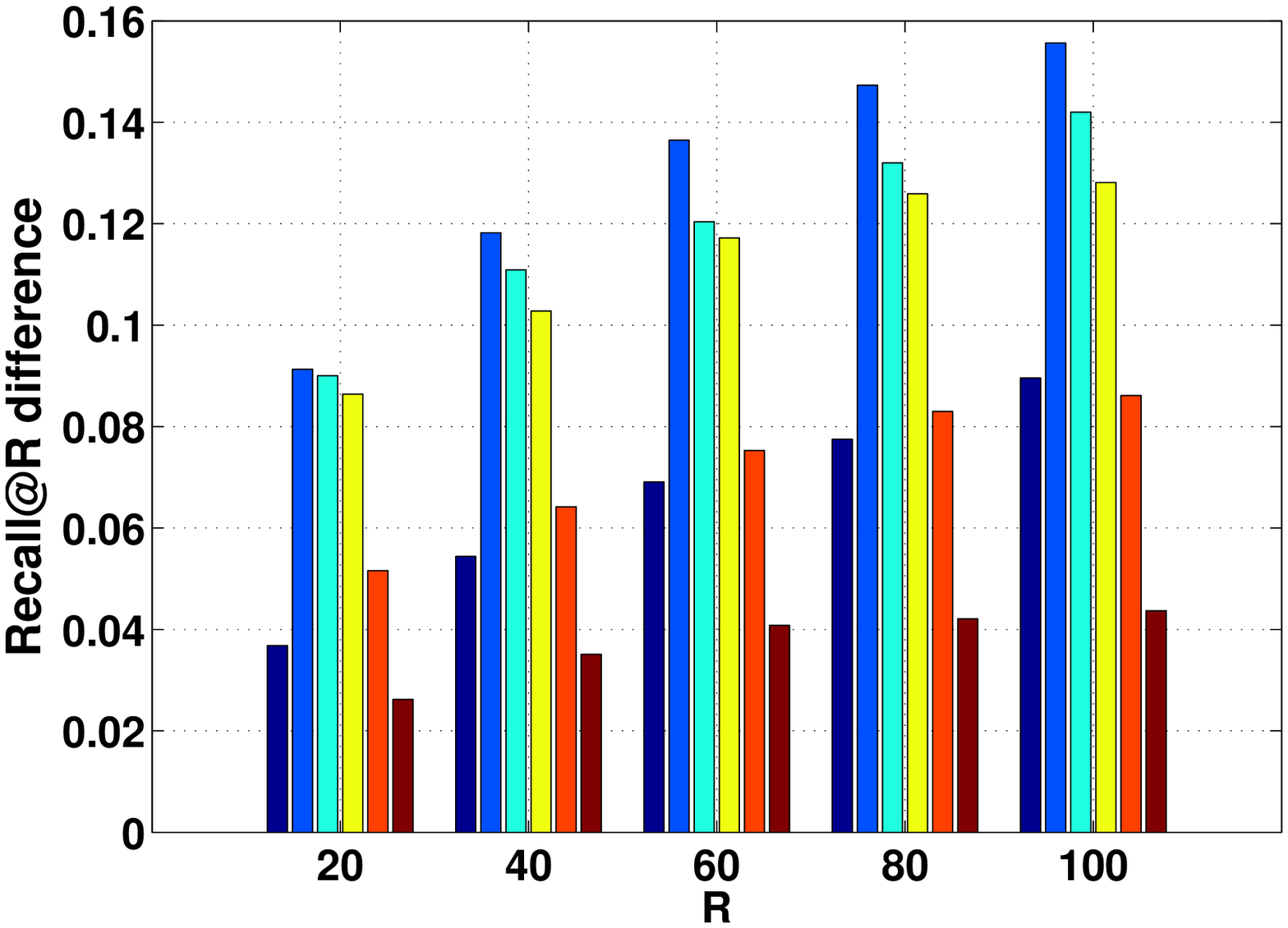} &
	\includegraphics[width=0.32\textwidth]{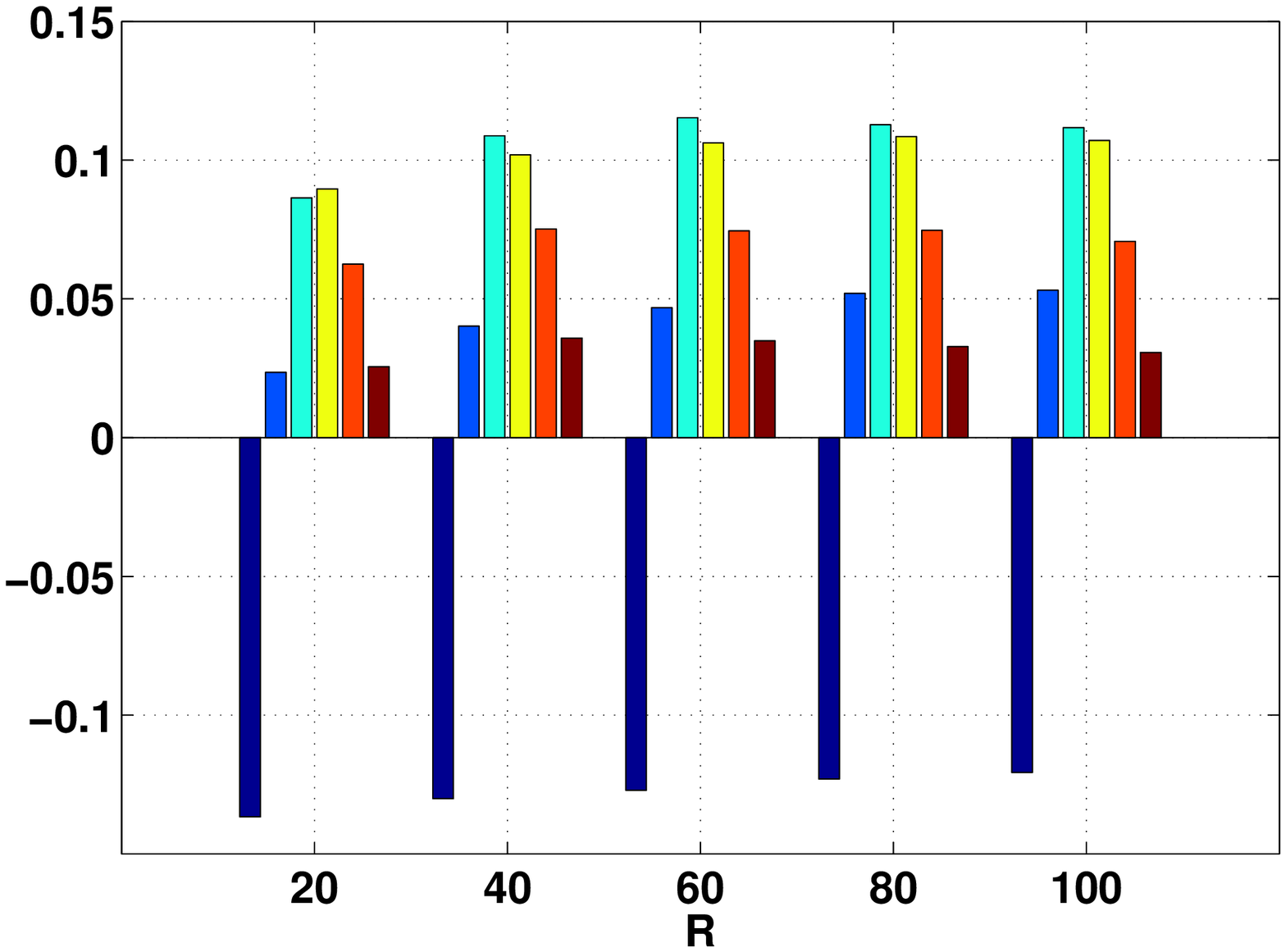} &
	\includegraphics[width=0.32\textwidth]{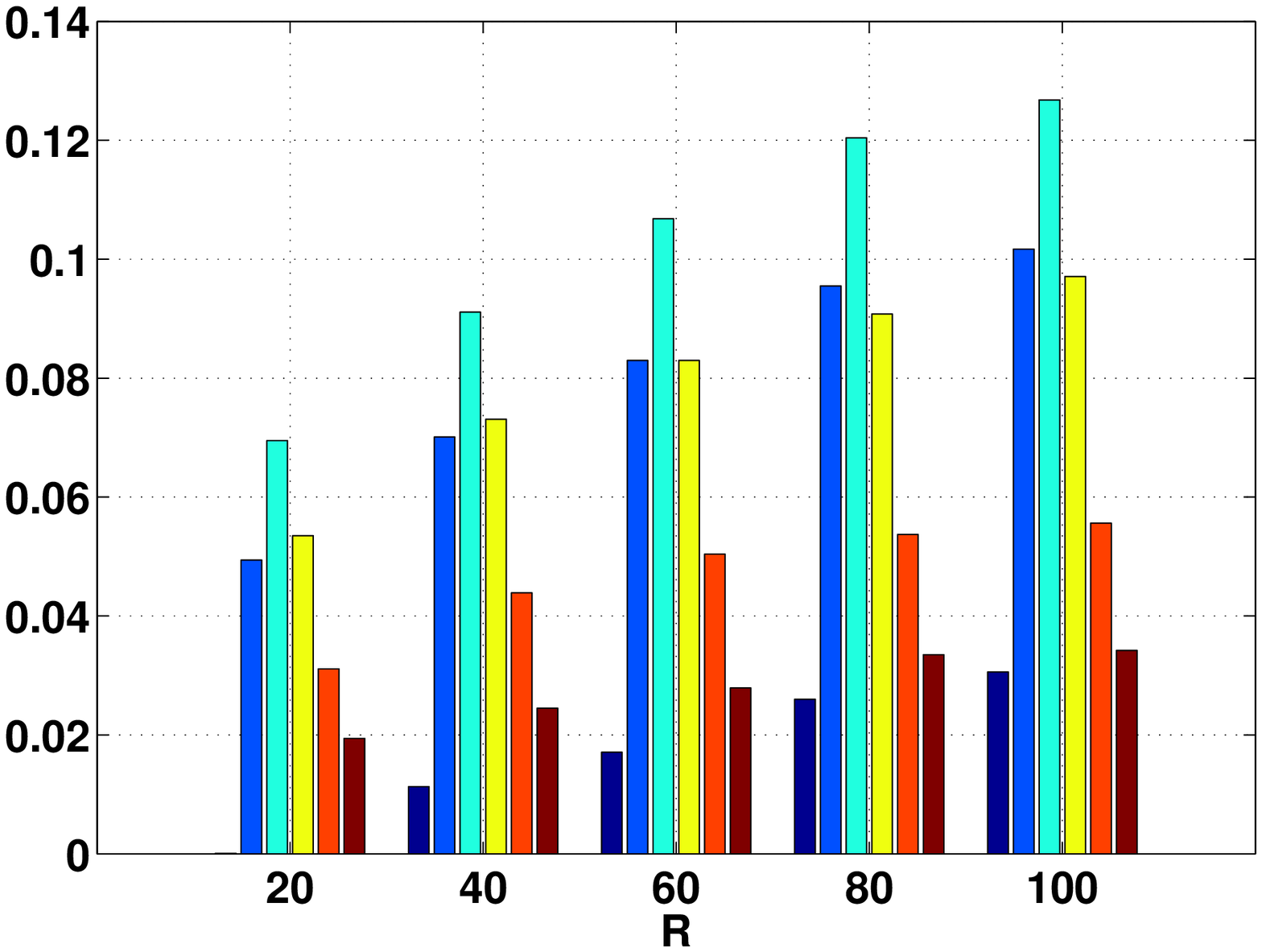} \\
	 & \text{(b) Intersection kernel} & 
\end{tabular}
\vspace{-0.1\baselineskip}
\caption{Recall@$R$ improvement over the vanilla KLSH with $m=1000$ for rank $\in\{16,32,64,128,256,512\}$. (Best viewed in color.)}
\label{fig:rank}
\end{figure*}

\begin{figure*}
	\centering
	\begin{subfigure}[b]{0.32\textwidth}
		\includegraphics[width=\textwidth]{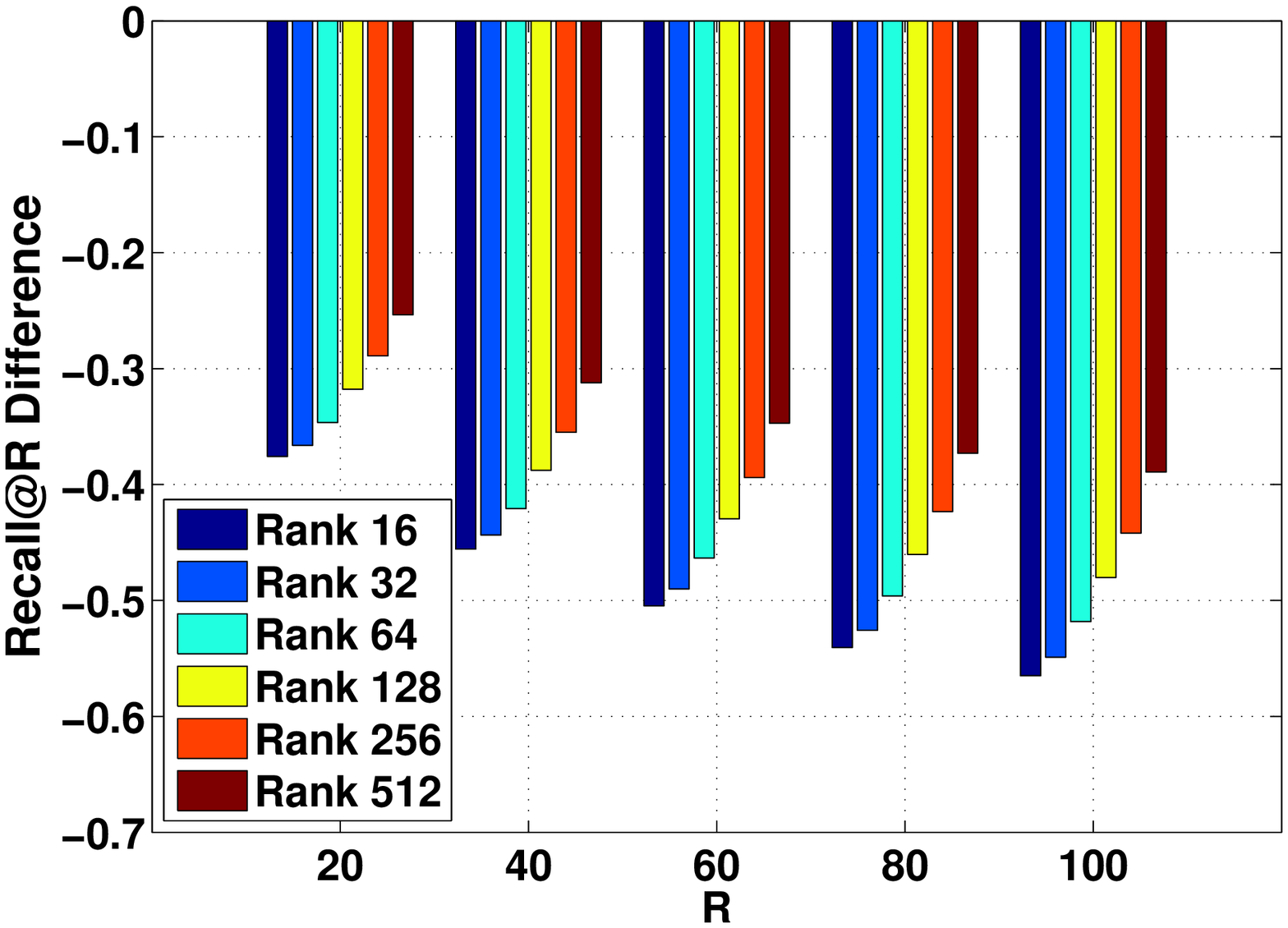}
		\vspace{-0.8\baselineskip}
		\caption{ }
		\label{fig:nystrom}
	\end{subfigure}
	\begin{subfigure}[b]{0.32\textwidth}
		\includegraphics[width=\textwidth]{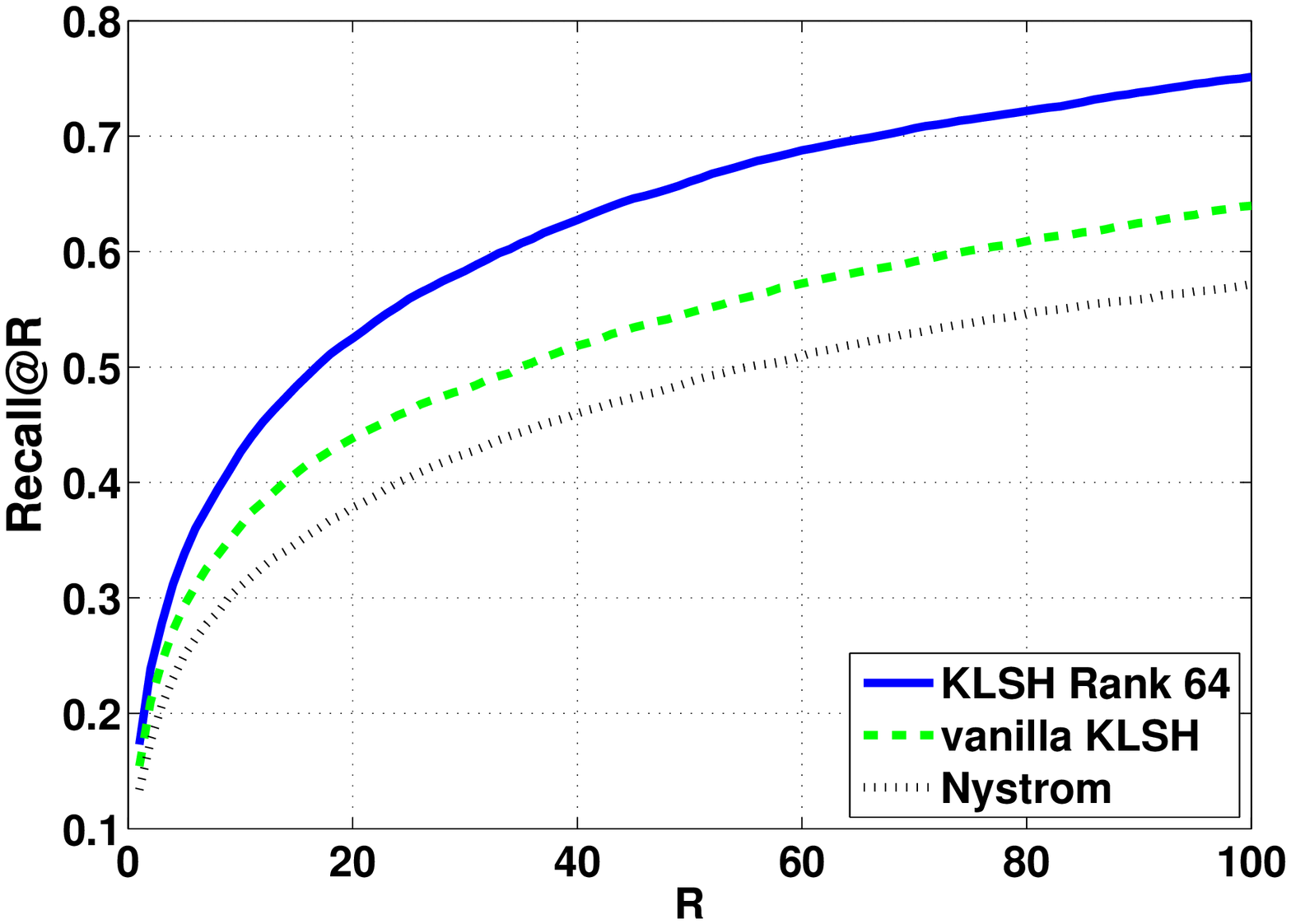}
		\vspace{-0.8\baselineskip}
		\caption{ }
		\label{fig:klsh_nystrom}
	\end{subfigure}
	\begin{subfigure}[b]{0.32\textwidth}
		\includegraphics[width=\textwidth]{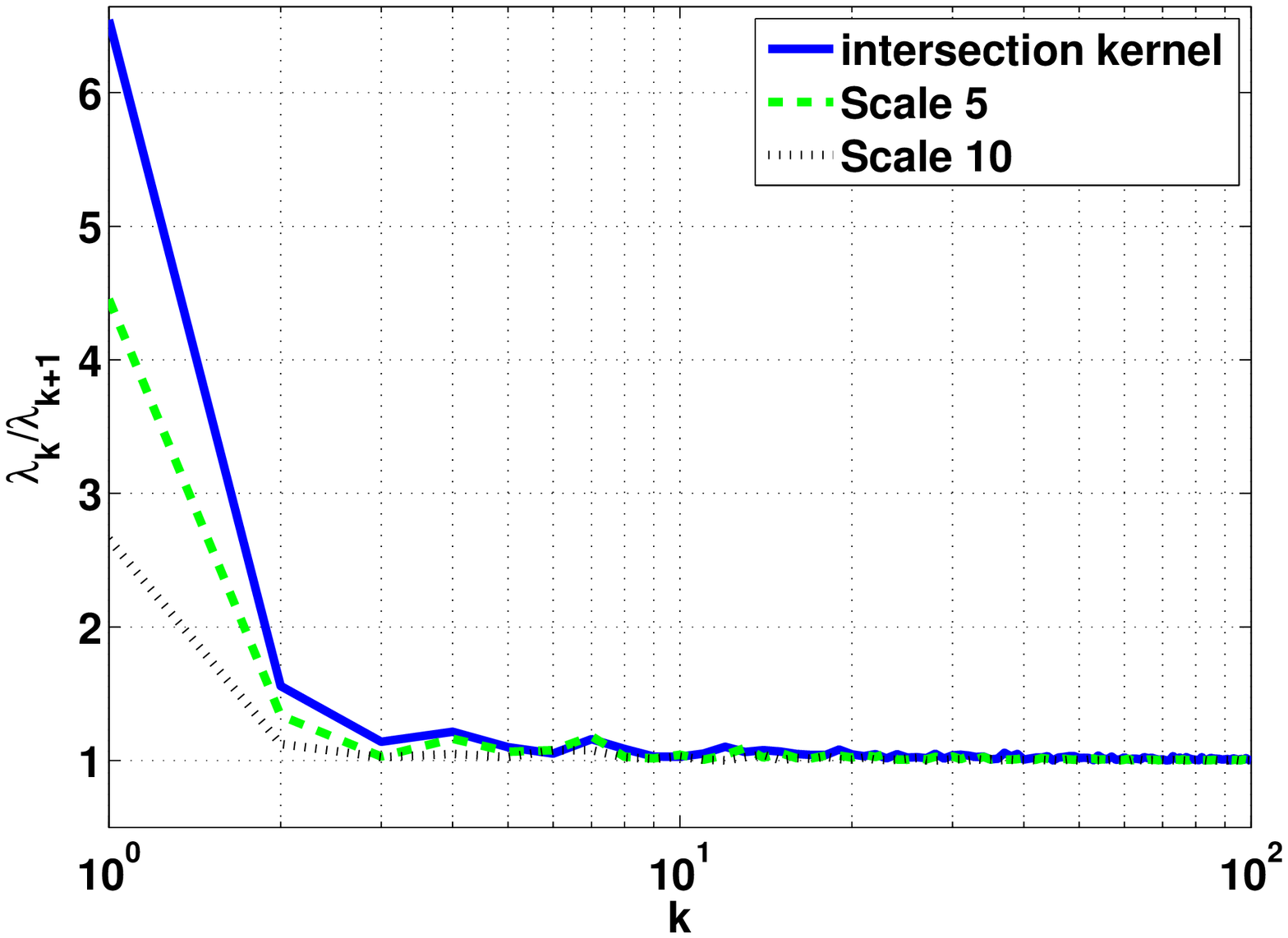}
		\vspace{-0.8\baselineskip}
		\caption{ }
		\label{fig:decay}
	\end{subfigure}
	\vspace{-0.1\baselineskip}
	\caption{SIFT1M intersection kernel. (a) Recall@$R$ improvement over the full-rank Nystr\"om method with $m=1000$ for rank $\in\{16,32,64,128,256,512\}$; (b) Recall@$R$ results for KLSH with rank 64, vanilla KLSH and the full-rank Nystr\"om method for $m=1000$; (c) Change of decaying properties for the exponential transformation with scale parameter $\in\{5,10\}$. (Best viewed in color.)}
\end{figure*}

\begin{figure*}[t]
\centering
\def\arraystretch{0.4}
\setlength{\tabcolsep}{3pt}
\begin{tabular}{ccc}
	\includegraphics[width=0.32\textwidth]{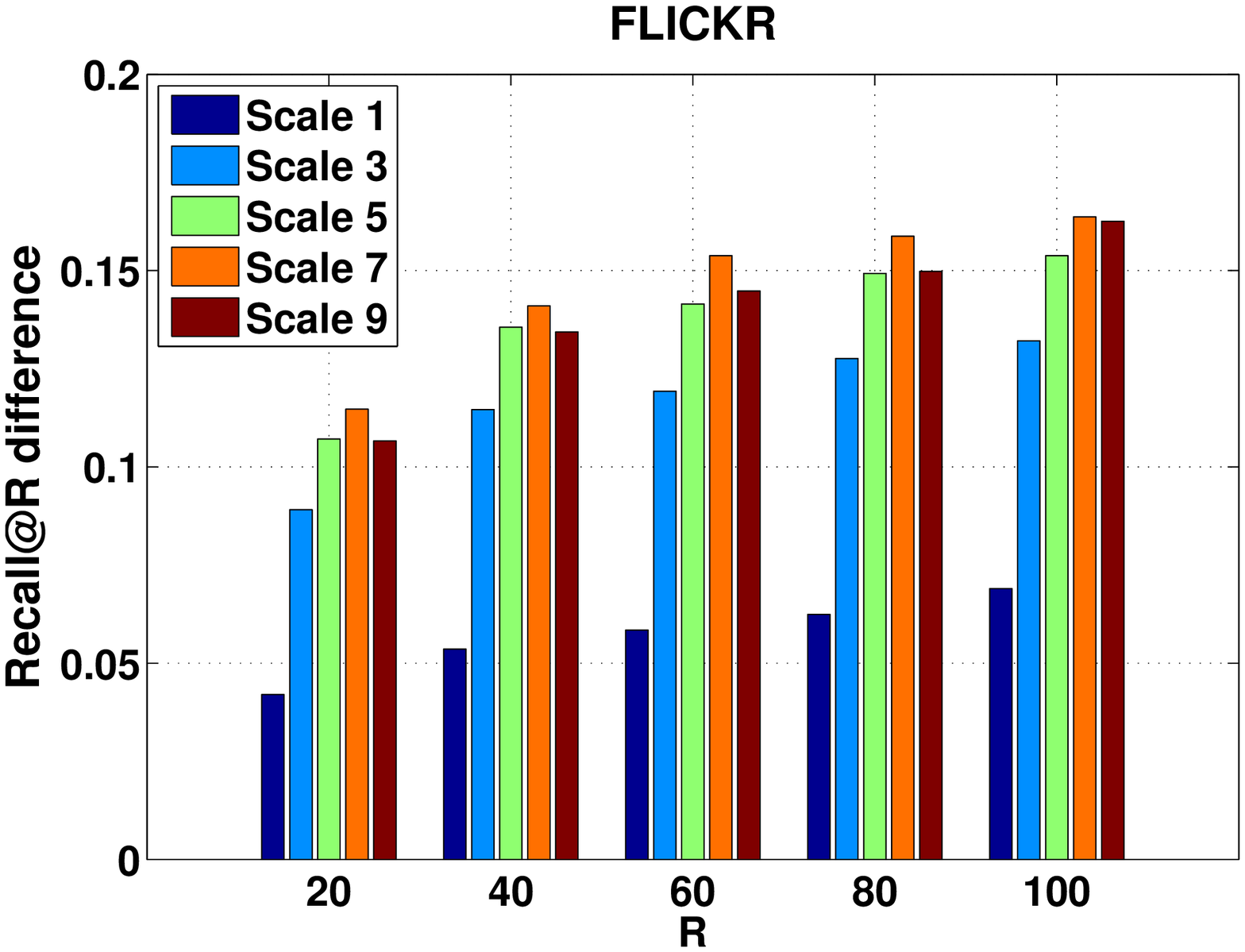} &
	\includegraphics[width=0.32\textwidth]{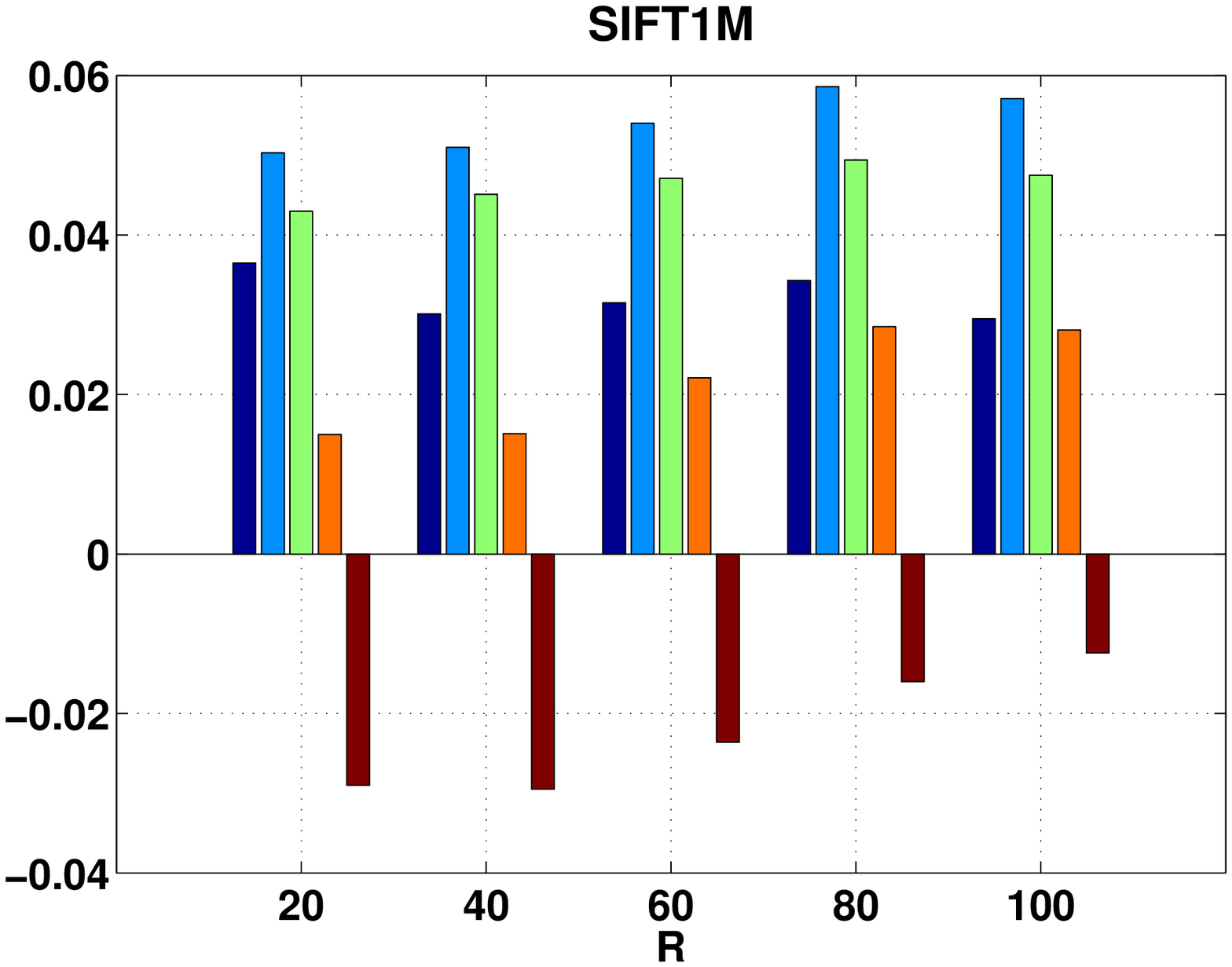} &
	\includegraphics[width=0.32\textwidth]{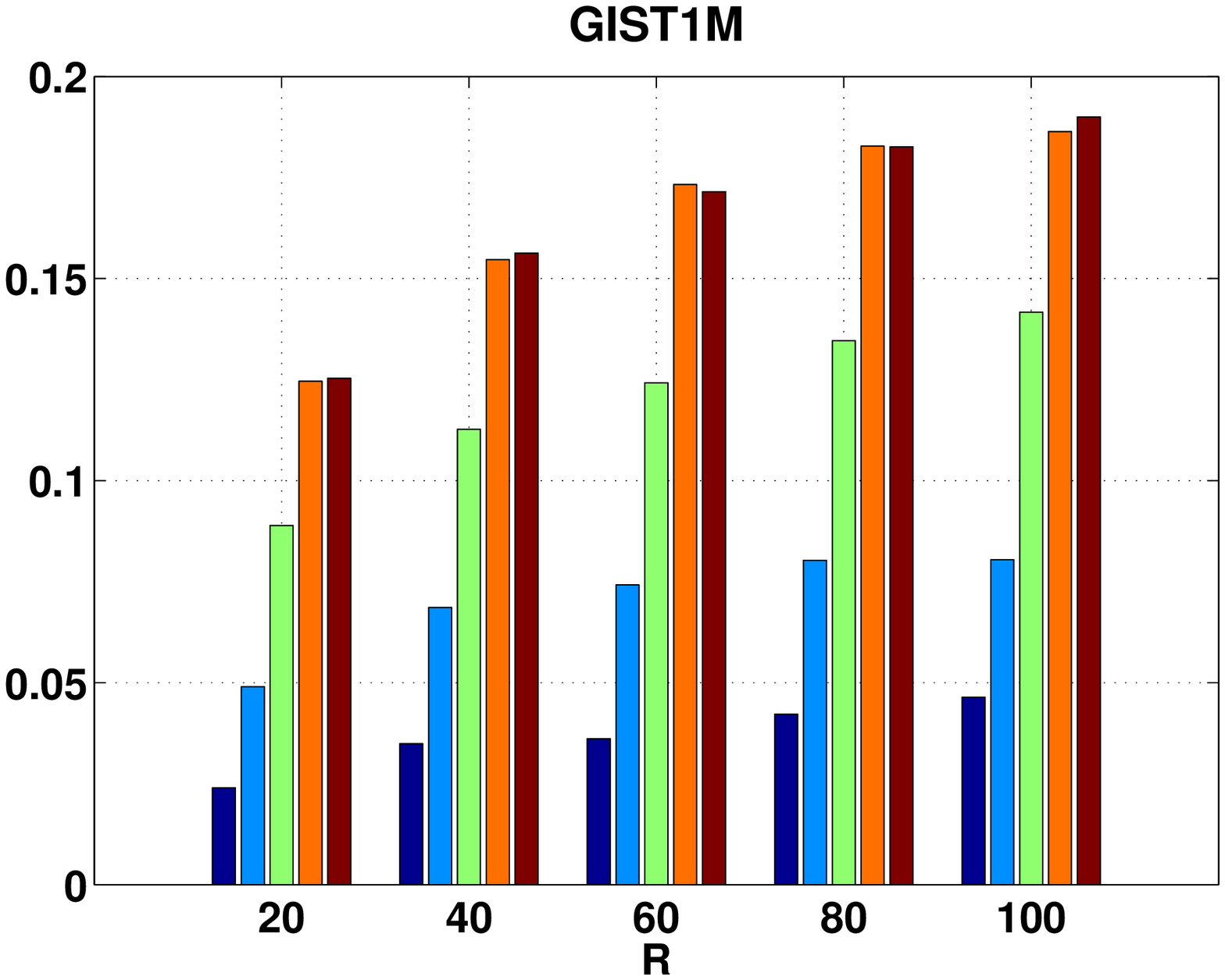} \\
	 & \text{(a) $\chi^2$ kernel} & \\
	\includegraphics[width=0.32\textwidth]{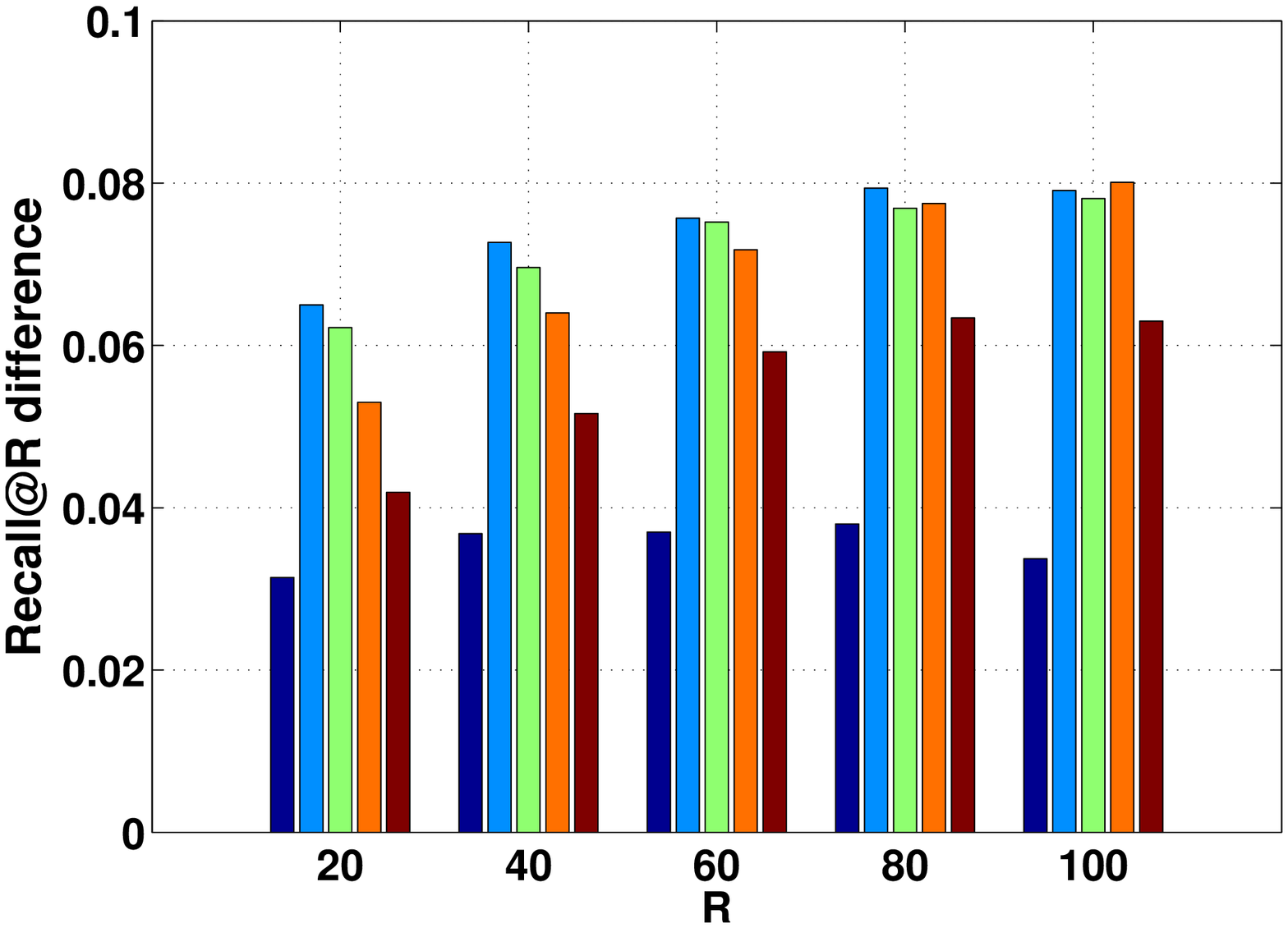} &
	\includegraphics[width=0.32\textwidth]{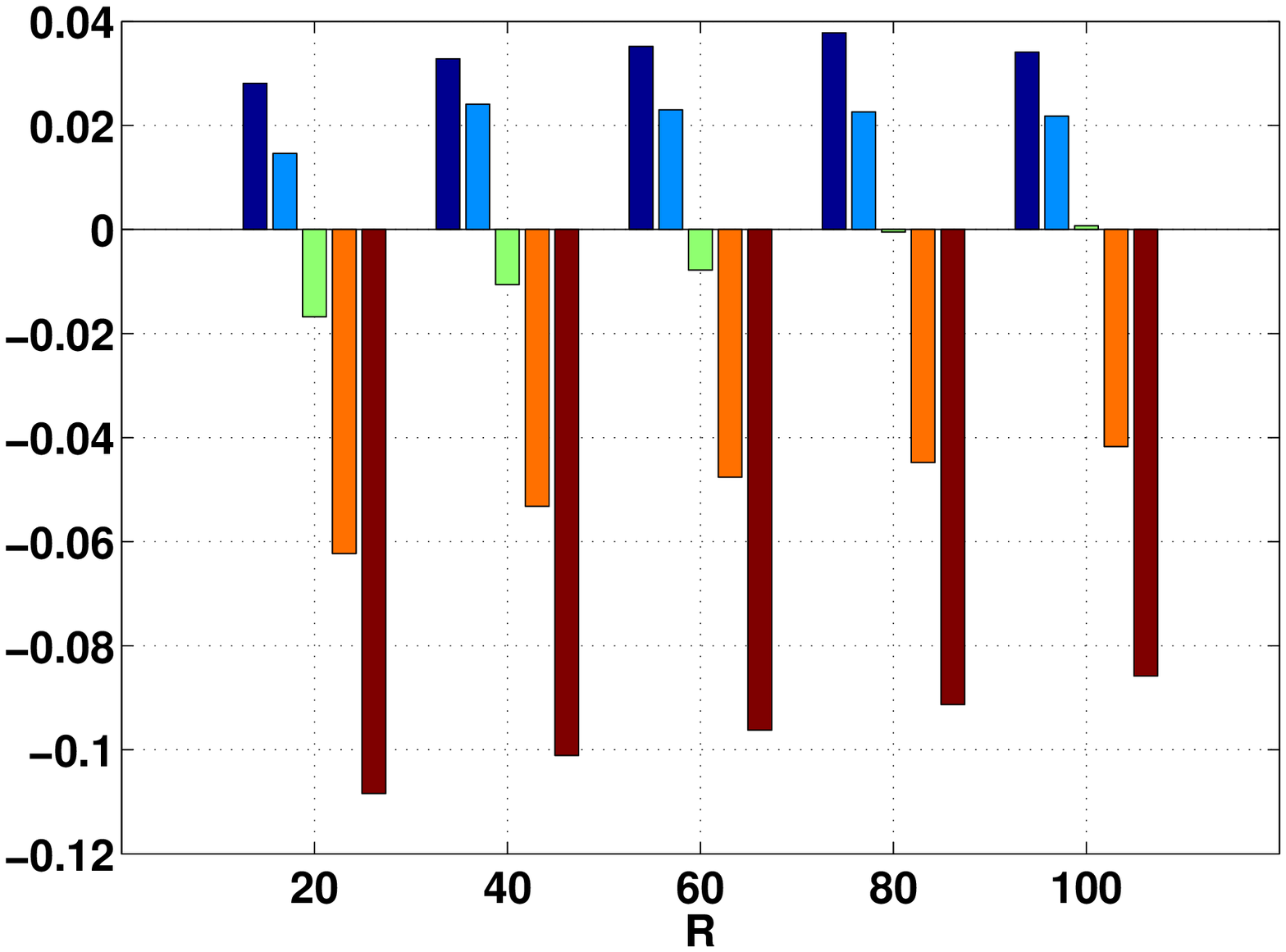} &
	\includegraphics[width=0.32\textwidth]{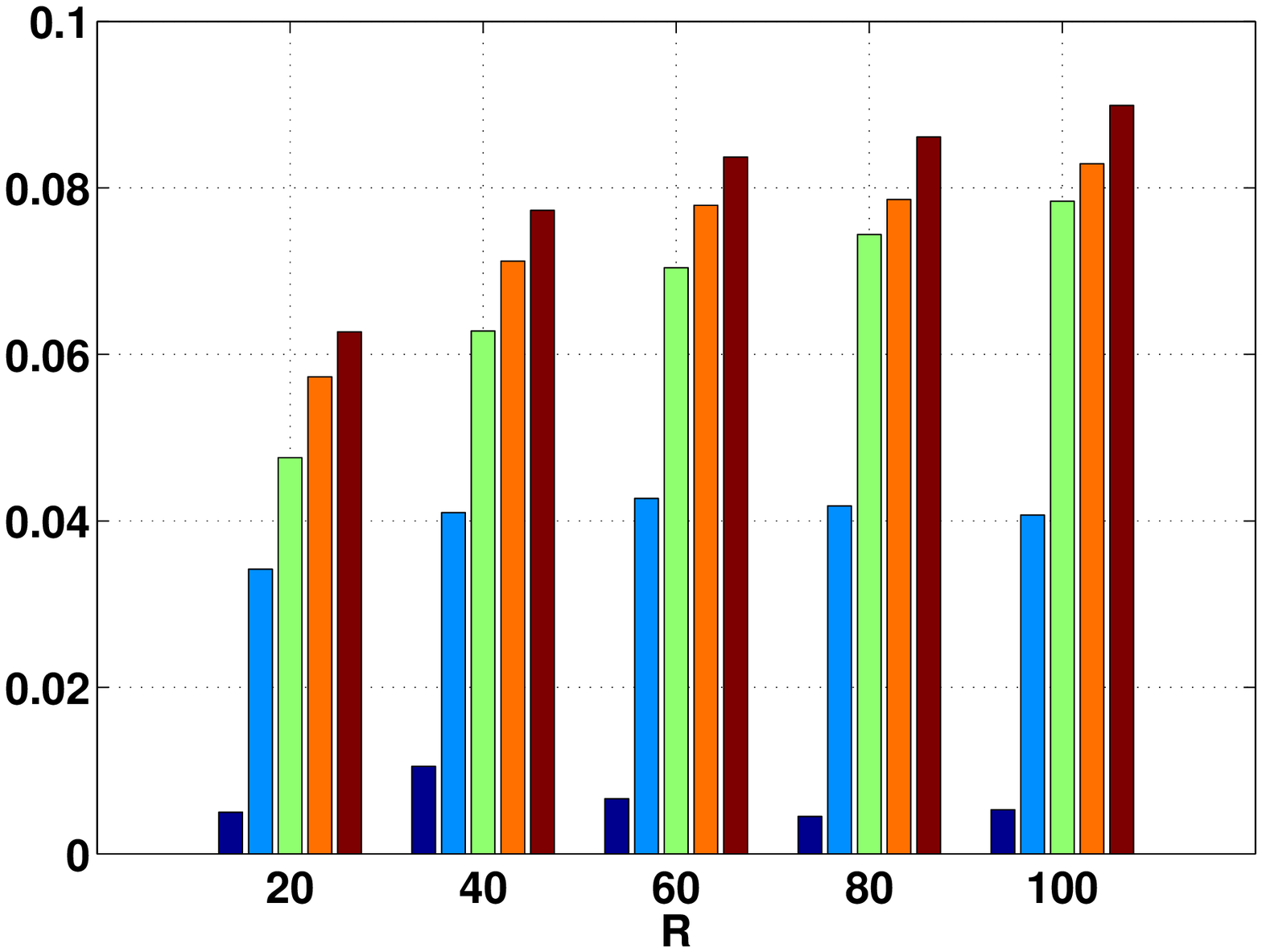} \\
	& \text{(b) Intersection kernel} & 
\end{tabular}
\vspace{-0.1\baselineskip}
\caption{Recall@$R$ result improvement over the low-rank KLSH with $m=1000$ for scale $\in\{1,3,5,7,9\}$. Here, we use rank $32$ for FLICKR, $100$ for SIFT1M, and $64$ for GIST1M. (Best viewed in color.)}
\label{fig:scale}
\end{figure*}

\begin{table*}[t]
	\centering
	\begin{tabular}{|c|c|c|c|c|c|}
	\hline
	Recall@100 & Dataset & KLSH & low-rank & low-rank+transformation & total improvement\\
	\hline
	$\chi^2$ kernel & Flickr & 0.3629 & 0.5125 & 0.6762 & $\bm{+0.3133}$ \\
	 & SIFT1M & 0.6942 & 0.7642 & 0.8213 & $\bm{+0.1271}$ \\
	 & GIST1M & 0.2252 & 0.3360 & 0.5260 & $\bm{+0.3008}$ \\
	\hline
	Intersection kernel & Flickr & 0.4182 & 0.5738 & 0.6529 & $\bm{+0.2347}$ \\
	 & SIFT1M & 0.6397 & 0.7468 & 0.7844 & $\bm{+0.1447}$ \\
	 & GIST1M & 0.2746 & 0.4014 & 0.4913 & $\bm{+0.2167}$ \\ \hline
	\end{tabular}
	\caption{Summary of absolute improvement for Recall@100. }
	\label{table:total}
\end{table*}

\subsection{Effect of rank}
Figure \ref{fig:rank} shows the effect of the rank $r$ (with all other parameters fixed): all but the smallest rank performed better or at least comparable with the vanilla KLSH. This further confirms the empirical results shown in \cite{kpca_lsh,aska} that KPCA+LSH with smaller number of principal components beats KLSH in retrieval performance. However, this is not the entire story. We can clearly see the performance tradeoff as discussed in Section~\ref{sec:tricks}. Initially, the retrieval performance improves with an increasing number of principal components used, which corresponds to decreasing $\lambda_r$. However, at some point, performance drops corresponding to the increase of $\delta_r$. For MIRFLICKR and GIST1M, the difference among ranks can be dramatic, showing the sensitivity of the choice of rank. 
In addition, the best-performing rank is not only dependent on the kernel chosen but also critically dependent on the dataset examined. 
Nonetheless, we can still obtain at least $7\%$ in absolute improvement for Recall@$100$ for the SIFT1M data which is least affected by this trade-off.
Here, the performance for different kernels is quite similar, but is divergent for different datasets. For instance, the optimal rank for the MIRFLICKR data is much smaller than that of the SIFT1M and GIST1M. However, we observe no relationship between rank and the number of bits used, making the recommendations of \cite{kpca_lsh} questionable.

\vspace{1mm}
\noindent\textbf{Comparison with the Nystr\"om method}. Figure~\ref{fig:nystrom} shows the effect of the rank $r$ for the Nystr\"om method: the performance is monotonically increasing with the rank. Moreover, it shows unacceptable retrieval performance even with rank $512$. Contrasting with the obvious tradeoff with the choice of ranks in KLSH, these results corroborate our earlier observation that KLSH is indeed different from the Nystr\"om method. Regarding the performance comparison, we can see from Figure~\ref{fig:klsh_nystrom} that the Nystr\"om method performs worse than both the low-rank version of KLSH and the  standard``full-rank" KLSH by a large margin.

\subsection{Effect of monotone transformation}
Here, we show the effect of the transformation introduced in \eqref{eqn:transformation}. Note that we are free to choose any possible transformation as long as the chosen transformation is increasingly monotonic; our choice of \eqref{eqn:transformation} is simply for illustration. We can see from Figure~\ref{fig:decay} how changing the scale parameter affects the decay of the eigenvalues. In particular, we see that increasing the scaling slows down the decay and will continue to decrease the decay as $s$ gets larger.

Figure~\ref{fig:scale} demonstrates the power of the transformation (with all other parameters fixed): the Recall@$R$ steadily increases as we slow down the decaying speed until the decaying is too gradual. And too large a $s$ may drop the performance significantly. The choice of $s$ and its usefulness also largely depends on both the kernel function and the dataset in consideration: it has more effect on the $\chi^2$ kernel than the intersection kernel. On the other hand, it is more effective in the GIST1M dataset than in SIFT1M. Note here, we are comparing with the original kernel with a fixed rank which favors the original kernel. Thus, there is room for further improvement by choosing a larger rank.

Table~\ref{table:total} summaries the total absolute improvement combining the two techniques together. We can see that the retrieval improvement is at least 12\%, sometimes much higher, among all benchmarks.This again validates the merit of our analysis in Section~\ref{sec:tricks} regarding the interesting trade-offs shown in our performance bound \eqref{eqn:klsh_bound} and demonstrates the power of these simple techniques.

\section{Conclusion}

We introduced a new interpretation of the kernelized locality-sensitive hashing technique. Our perspective makes it possible to circumvent the conceptual issues of the original algorithm and provides firmer theoretical ground by viewing KLSH as applying LSH on appropriately projected data. This new view of KLSH enables us to prove the first formal retrieval performance bounds, which further suggests two simple techniques for boosting the retrieval performance. We have successfully validated these results empirically on large-scale datasets and showed that the choice of the rank and the monotone transformation are vital to achieve better performance. 

\bibliographystyle{unsrt}
\bibliography{klsh_arxiv}

\section*{Appendix}

We first present a proof of Lemma 2.
\begin{proof}
By the Pythagorean theorem, we have
\[
	N(\bx)^2 = \|P_{\hat{V}_k}(\Phi(\bx))\|^2 = \|\Phi(\bx)\|^2 - \|\hpp(\Phi(\bx))\|^2.
\]
The residual $\hpp(\Phi(\bx))$ can be further decomposed into
\begin{equation}
\label{eqn:hpp_eq}
	\hpp(\Phi(\bx)) = \pp(\Phi(\bx)) + \left(\hpp(\Phi(\bx)) - \pp(\Phi(\bx))\right).
\end{equation}
For the first term, we have
\[
	\|\pp(\Phi(\bx))\| \leq \sqrt{\lambda_k}.
\]
Then applying Theorem 4 in~\cite{pca_convergence}, with probability at least $1-e^{-\xi}$, we can also bound the second part of~\eqref{eqn:hpp_eq}:
\[
\left\| \hpp(\Phi(\bx)) - \pp(\Phi(\bx)) \right\| \leq \frac{2M}{\delta_k\sqrt{m}}\left(1+\sqrt{\frac{\xi}{2}}\right),
\]
where $\delta_k = \frac{\lambda_k-\lambda_k+1}{2}$ and $M = \sup_{\bx}\kappa(\bx,\bx) = 1$. Thus,
\[
\|\hpp(\Phi(\bx))\| \leq \sqrt{\lambda_k}+\frac{2}{\delta_k\sqrt{m}}\left(1+\sqrt{\frac{\xi}{2}}\right).
\]
Putting these pieces together, with probability at least $1-e^{-\xi}$, we have
\begin{align}
N(\bx) =& \sqrt{1-\|\hpp(\Phi(\bx))\|^2} \geq 1-\|\hpp(\Phi(\bx))\| \nonumber \\
\ge &1-\sqrt{\lambda_k}-\frac{2}{\delta_k\sqrt{m}}\left(1+\sqrt{\frac{\xi}{2}}\right). \nonumber
\end{align}
\end{proof}

Our proof of Theorem 3, which is given below, requires a few prerequisite results, which we briefly summarize now.
The first is regarding the upper bound of inner product of complement of projections onto the subspace from kernel principal component analysis.
\begin{lemma}
\label{lemma:res_inner}
Consider a feature map $\Phi \in \mathcal{H}$ defined by a normalized kernel function $\kappa(\cdot, \cdot)$ in $\mathcal{X}$ with a probability measure $p$. Let $S_{m}=\{\bx_1,\dots, \bx_m\}$ be $m$ i.i.d. samples drawn from $p$ and $C$ be the covariance operator of $p$ with decreasing eigenvalues $\lambda_1\geq\lambda_2\geq \dots$. Let $V_k$ and $\hat{V}_k$ be the eigen-spaces corresponding the covariance operator $C$ and its empirical counterpart $C_{S_m}$. Then, with probability at least $1-e^{-\xi}$ over the selection of $S_m$, we have
\[
	\langle \hpp(\Phi(\bx)), \hpp(\Phi(\by))\rangle \le \left( \sqrt{\lambda_k}+\frac{2}{\delta_k\sqrt{m}}\left(1+\sqrt{\frac{\xi}{2}}\right) \right)^2. 
\]
\end{lemma}

\begin{proof}
From Cauchy-Schwarz inequality, we have
\[
	\langle \hpp(\Phi(\bx)), \hpp(\Phi(\by))\rangle \leq \|\hpp(\Phi(\bx)) \| \|\hpp(\Phi(\by))\|.
\]
Then, for any $x\in \mathcal{X}$,
\begin{align}
	&\| \hpp(\Phi(\bx))\| \nonumber \\
	\le& \| \pp(\Phi(\bx))\| + \|\hpp(\Phi(\bx)) - \pp(\Phi(\bx))\| \nonumber \\
	\le& \| \pp(\Phi(\bx))\| + \|\hpp- \pp\| \|\Phi(\bx)\|. \nonumber
\end{align}
By the definition of operator norms, we have $\| \pp(\Phi(\bx))\| \leq \sqrt{\lambda_k} \|\Phi(\bx)\|$. Moreover, as stated in Theorem 4 in~\cite{pca_convergence}, with probability at least $1-e^{-\xi}$, we have that 
\[
	\|\hpp-\pp\| \leq \frac{2M}{\delta_k\sqrt{m}}\left(1+\sqrt{\frac{\xi}{2}}\right)
\]
where $\delta_k = \frac{\lambda_k-\lambda_{k+1}}{2}$ and $M = \sup_x\kappa(\bx, \bx) = 1$.

Hence, with probability at least $1-e^{-\xi}$ over $S_m$, we have
\[
	\langle \hpp(\Phi(\bx)), \hpp(\Phi(\by))\rangle \le \left( \sqrt{\lambda_k}+\frac{2}{\delta_k\sqrt{m}}\left(1+\sqrt{\frac{\xi}{2}}\right) \right)^2  \|\Phi(\bx)\| \|\Phi(\by)\|.
\]

Since $\|\Phi(\bx)\| = 1$ for any $\bx$, we have proved the lemma.
\end{proof}

For completeness, we also state the performance bound of standard LSH:
\begin{theorem}
\cite{lsh, plsh, simhash}. Let $(X,d_{X})$ be a metric space on a subset of $\mathbb{R}^d$. Suppose that $(X,d_{X})$ admits a similarity hashing family. Then for any $\epsilon>0$, there exists a randomized algorithm for $(1+\epsilon)$-near neighbor on $n$-point database with success probability larger than 0.5, which uses $O(dn+n^{1+\frac{1}{1+\epsilon}})$ space, with query time dominated by $O(n^{\frac{1}{1+\epsilon}})$ distance computations.
\label{thm:lsh}
\end{theorem}

Using the above results, we are ready to prove our main result in Theorem 3.
\begin{proof}
By the definition of $P_{\hat{V}_k}$, we can decompose $\kappa(\bq, \yqk)$ into two parts,
\begin{equation}
\label{eqn:eq1}
\kappa(\bq, \yqk) = \hk(\bq, \yqk) + \langle \hpp(\Phi(\bq)), \hpp(\Phi(\yqk)) \rangle.
\end{equation}
Thus, by Lemma~\ref{lemma:res_inner}, we have
\begin{align}
 \kappa(\bq, \yqk) \ge & \hk(\bq, \yqk) - \left|\langle \hpp(\Phi(\bq)), \hpp(\Phi(\yqk)) \rangle\right| \nonumber \\
\ge & \hk(\bq, \yqk) - \left( \sqrt{\lambda_k}+\frac{2}{\delta_k\sqrt{m}}\left(1+\sqrt{\frac{\xi}{2}}\right) \right)^2. \nonumber
\end{align}
To lower-bound $\hk(\bq, \yqk)$, we need to use the result for LSH, which asks for normalized kernels. Thus, we consider the normalized version,
\[
	\hk_n(\bq, \yqk) = \frac{\hk(\bq, \yqk)}{N(\bq)N(\yqk)}.
\]
Then we can relate a distance function via $\hat{d}(\bq, \yqk) = 1-\hk_n(\bq, \yqk)$ \cite{simhash}. By the LSH guarantee in Theorem~\ref{thm:lsh}, with probability larger than $0.5$, we have
\[
	\hat{d}(\bq, \yqk) \leq (1+\epsilon)\hat{d}(\bq, \oyqk),
\]
which is equivalent to 
\[
	\hk_n(\bq, \yqk) \geq (1+\epsilon)\hk_n(\bq, \oyqk) - \epsilon,
\]
where $\oyqk = \text{argmax}_{\bm{x}\in S}\hk_n(\bq,\bm{x})$. Applying Lemma 2, with probability $1-e^{-\xi}$, the true optimal $\oyq$ with respect to $\kappa$ is not eliminated for LSH, thus we have $\hk_n(\bq, \oyqk) \geq \hk_n(\bq, \oyq)$ due to the optimality of $\oyqk$ with respect to $\hk_n$, and with probability $0.5\times (1-e^{-\xi})$
\[
	\hk_n(\bq, \yqk) \geq (1+\epsilon)\hk_n(\bq, \oyq) - \epsilon.
\]
Expanding $\hk_n$, we get
\[
	\frac{\hk(\bq, \yqk)}{N(\bq)N(\yqk)} \geq (1+\epsilon)\frac{\hk(\bq, \oyq)}{N(\bq)N(\oyq)} - \epsilon.
\]
which can be reduced to
\[
	\hk(\bq, \yqk) \geq (1+\epsilon)(1-\sqrt{\lambda_k}-\eta)\hk(\bq, \oyq) - \epsilon,
\]
since $1-\sqrt{\lambda_k}-\eta \leq N(\bx) \leq 1$.
Decompose $\hk(\bq, \oyq)$ on right hand-side above as
\[
	\hk(\bq, \oyq) = \kappa(\bq, \oyq) - \langle \hpp(\Phi(\bq)), \hpp(\Phi(\oyq))\rangle.
\]
Thus,
\begin{equation}
\label{eqn:eq2}
\hk(\bq, \yqk) \ge (1+\epsilon)(1-\sqrt{\lambda_k}-\eta)\kappa(\bq, \oyq) -(1+\epsilon)\left|\langle \hpp(\Phi(\bq)), \hpp(\Phi(\oyq))\rangle\right| - \epsilon. 
\end{equation}
Combining results in Equation~\eqref{eqn:eq1} and~\eqref{eqn:eq2},
\begin{align}
\kappa(\bq, \yqk) \ge& \hk(\bq, \yqk) - \left|\langle \hpp(\Phi(\bq)), \hpp(\Phi(\yqk)) \rangle\right| \nonumber \\
\ge & (1+\epsilon)(1-\sqrt{\lambda_k}-\eta)\kappa(\bq, \oyq) -\epsilon \nonumber \\
&-(1+\epsilon)\left|\langle \hpp(\Phi(\bq)), \hpp(\Phi(\oyq))\rangle\right| - \left|\langle \hpp(\Phi(\bq)), \hpp(\Phi(\yqk)) \rangle\right| \nonumber
\end{align}
Now we can apply Lemma~\ref{lemma:res_inner}:
\[
\kappa(\bq, \yqk) \ge  (1+\epsilon)(1-\sqrt{\lambda_k}-\eta)\kappa(\bq, \oyq) - \epsilon -(2+\epsilon) \left( \sqrt{\lambda_k}+\eta \right)^2.
\]
\end{proof}
\end{document}